%% file: main.tex
\documentclass{article}

\usepackage[nonatbib,final]{neurips_2024}
\usepackage[style=alphabetic,natbib=true]{biblatex}
\bibliography{refs}

\usepackage[utf8]{inputenc} 
\usepackage[T1]{fontenc}    
\usepackage{url}            
\usepackage{booktabs}       
\usepackage{nicefrac}       
\usepackage{microtype}      
\usepackage{xcolor}         
\usepackage{graphicx}
\usepackage{amsmath}
\usepackage{amssymb}
\usepackage{amsthm,amsfonts}
\usepackage{mathtools} 
\usepackage{epsfig}
\usepackage[hidelinks]{hyperref}
\usepackage{enumerate}

\hypersetup{
    colorlinks=true,
    linkcolor=[rgb]{0.0,0.0,0.9},
    filecolor=[rgb]{0.0,0.0,0.9},      
    urlcolor=[rgb]{0.0,0.0,0.9},
    citecolor=[rgb]{0.0,0.0,0.9},
    pdftitle={Almost Free: Self-concordance in Natural Exponential Families and an Application to Bandits},
    pdfpagemode=FullScreen,
    }

\author{%
  Shuai Liu\\
  University of Alberta\\
  \texttt{shuailiu725@gmail.com} \\
  \And
  Alex Ayoub \\
  University of Alberta \\
  \texttt{aayoub@ualberta.ca} \\
  \AND
  Flore Sentenac \\
  HEC Paris\\
  \texttt{sentenac@hec.fr} \\
  \And
  Xiaoqi Tan \\
  University of Alberta\\
  \texttt{xiaoqi.tan@ualberta.ca}
  \And
  Csaba Szepesvári \\
  University of Alberta \\
  \texttt{csaba.szepesvari@gmail.com}
}

\usepackage{xspace}
\usepackage{thmtools} 
\usepackage{thm-restate}

\usepackage{algorithm,algorithmic}

\usepackage[capitalise]{cleveref}
\crefname{lemma}{Lemma}{Lemmas}
\crefname{assumption}{Assumption}{Assumptions}
\crefname{theorem}{Theorem}{Theorems}
\newtheorem{theorem}{Theorem}
\newtheorem{corollary}[theorem]{Corollary}
\newtheorem{lemma}[theorem]{Lemma}
\newtheorem{example}[theorem]{Example}

\newtheorem{proposition}[theorem]{Proposition}
\newtheorem{definition}{Definition}

\newtheorem{remark}{Remark}

\newtheorem{assumption}{Assumption}
\newcounter{subassumption}[assumption]
\renewcommand{\thesubassumption}{(\textit{\roman{subassumption}})}
\makeatletter
\renewcommand{\p@subassumption}{\theassumption}
\makeatother
\newcommand{\subasu}{
  \refstepcounter{subassumption}%
  \thesubassumption~\ignorespaces}

\DeclareMathOperator*{\argmax}{arg\,max}
\DeclareMathOperator*{\argmin}{arg\,min}

\DeclareMathOperator*{\Var}{Var}
\DeclareMathOperator*{\var}{Var}

\newcommand{\bS}{\mathbf{S}}

\newcommand{\E}{\mathbb{E}}
\newcommand{\EE}{\mathbb{E}}
\newcommand{\R}{\mathbb{R}}
\newcommand{\RR}{\mathbb{R}}
\newcommand{\mbR}{\mathbb{R}}

\newcommand{\PP}{\mathbb{P}}
\newcommand{\NN}{\mathbb{N}}
\newcommand{\II}{\mathbb{I}}
\newcommand{\FF}{\mathbb{F}}
\newcommand{\cF}{\mathcal{F}}

\newcommand{\cL}{\mathcal{L}}
\newcommand{\cC}{\mathcal{C}}

\newcommand{\cU}{\mathcal{U}}
\newcommand{\cV}{\mathcal{V}}

\newcommand{\cX}{\mathcal{X}}

\newcommand{\cD}{\mathcal{D}}

\newcommand{\cQ}{\mathcal{Q}}
\newcommand{\mfC}{\mathfrak{C}}

\newcommand{\norm}[1]{\left \lVert #1 \right \rVert}
\newcommand{\spaced}[1]{\quad \text{#1} \quad}

\newcommand{\Prob}[1]{\mathbb{P}(#1)}

\newcommand{\CGF}{CGF\xspace}
\newcommand{\MGF}{MGF\xspace}
\newcommand{\NEF}{NEF\xspace}
\newcommand{\selfconc}{\Gamma} 
\newcommand{\regret}{\mathrm{Regret}\xspace}
\newcommand{\scfunc}{\Gamma} 

\newcommand{\cG}{\mathcal{G}}
\newcommand{\ER}{\mathcal{E}_{\textrm{right}}}
\newcommand{\EL}{\mathcal{E}_{\textrm{left}}}

\usepackage[size=tiny,disable]{todonotes}
\newcommand{\todoc}[2][]{\todo[size=\scriptsize,color=red!20!white,#1]{CS: #2}}
\newcommand{\todoa}[2][]{\todo[size=\scriptsize,color=blue!20!white,#1]{AA: #2}}
\newcommand{\todos}[2][]{\todo[size=\scriptsize,color=green!20!white,#1]{SL: #2}}

\title{Almost Free: Self-concordance in Natural Exponential Families and an Application to Bandits}
\begin{document}

\maketitle
\input{main/abstract}
\input{main/introduction}
\input{main/setting}

\input{main/main_result}

\input{main/OFU-GLB}

\input{main/conclusion}

\begin{ack}
The authors would like to thank Iosif Pinelis for his helpful insights to our work and Samuel Robertson for reviewing and earlier version of this manuscript. Csaba Szepesvári also gratefully acknowledges funding from the
Canada CIFAR AI Chairs Program, Amii and NSERC.
\end{ack}

\clearpage
\todoc[inline]{Someone should tidy up the bibliography.}
\printbibliography
\appendix

\clearpage
\input{appendix/self_concordance_thm}
\input{appendix/subgaussian}
\input{appendix/bandits}
\input{appendix/aux_lemmas}

\end{document}

%% file: main/abstract.tex
\begin{abstract}
    We prove that single-parameter natural exponential families with subexponential tails are self-concordant with polynomial-sized parameters. For subgaussian natural exponential families we establish an exact characterization of the growth rate of the self-concordance parameter. Applying these findings to bandits allows us to fill gaps in the literature: We show that optimistic algorithms for generalized linear bandits enjoy regret bounds that are both second-order (scale with the variance of the optimal arm's reward distribution) and free of an exponential dependence on the bound of the problem parameter in the leading term. To the best of our knowledge, ours is the first regret bound for generalized linear bandits with subexponential tails, broadening the class of problems to include Poisson, exponential and gamma bandits.
\end{abstract}

%% file: main/introduction.tex
\section{Introduction}\label{section:intro}
Single-parameter natural exponential families (NEFs) \citep{morris1982} 
abound in statistical applications \citep{McCullagh:1989,brown1986,wainwright2008graphical,morris2009}.
In this paper we study properties of \NEF{}s and in doing so we make two main contributions:
\begin{enumerate}
    \item 
    We study how tail properties of the base distribution of a NEF impose limits on the NEF: if the base distribution is subexponential (subgaussian), we show that the NEF is \emph{self-concordant} with a stretch factor that grows inverse quadratically (respectively, linearly)
    \item In generalized linear 
    bandits whose reward distributions follow a NEF with subexponential base distribution, we show how this new result 
    can be utilized to derive a novel second order regret bound whose leading term is free of exponential dependencies on the problem parameter --- the first such result for this setting. 
\end{enumerate}

The class of distributions our results extend to includes: normal, Poisson, exponential, gamma and negative binomial distributions. Our findings \textit{partially} address conjectures on whether generalized linear models with unbounded targets are self-concordant \citep{Bach2010self,faury2020improved}. This significantly generalizes the case when the targets are assumed to be bounded\footnote{This assumption does not hold for distributions with unbounded support such as: normal, Poisson, exponential, gamma and negative binomial distributions.} \citep{pmlr-v99-marteau-ferey19a,Ostrovskii2020finite,russac2021self} and thus extends the applicability of the results therein. \todoa{check/polish this sentence?} 

Self-concordance in NEFs turns out to be useful for both optimization and statistical estimation. The self-concordance property controls the remainder term, or approximation error, of a NEF's second-order Taylor expansion. This is useful in designing and analyzing both estimation and optimization methods. Historically the self-concordance property was first found to be useful for optimization \citep{nesterov1989self,nemirovski2008interior,sun2019generalized} and later for statistical estimation \citep{Bach2010self,pmlr-v99-marteau-ferey19a,Ostrovskii2020finite,bilodeau2023minimax}. In this paper, we will employ the self-concordance property in bandit problems \citep{Lattimore_Szepesvári_2020} where it helps with controlling the error terms related to estimation. 


Generalized linear bandits (GLBs) \citep{Filippi2010parametric} has emerged as a standard framework for studying the role that nonlinear function approximation plays in decision making problems. In the special case of logistic bandits, \cite{faury2020improved} were the first to exploit the self-concordance property of the Bernoulli distribution\todoa{I guess this is correct enough?}
\todoc{Taking out the footnote: ``more specifically the NEF whose base distribution is a Bernoulli distribution''. Why do we have this here?}
\todoa{It was pointed out to me that this statement is not technically correct? }in order to get regret bounds free of an exponential dependence on the size of the problem parameter\todoa{if anyone has another word for dependencies let me know} in the leading term. Instead of naively approximating the nonlinear sigmoid function with a linear first order Taylor expansion (which would incur the aforementioned exponential term)\todoc{I bet the reader has no idea what this dependency is on. So we should say it upfront.}\todoa{I added the phrase ``size of the problem parameter'' as was done in contribution list.} \citep{pmlr-v70-li17c,jun2017scalable,saha2023dueling}, \cite{faury2020improved} use self-concordance to get a tighter second order 
Taylor expansion that better captures the curvature of the sigmoid function. Employing improved self-concordant analysis, \cite{Abeille2020InstanceWiseMA} get second-order regret bounds for logistic bandits and \cite{janz2023exploration} extend these results to GLBs, under the assumption that the underlying reward distributions 
are self-concordant. We build upon this line of research and fill a gap in the bandit literature by designing and analyzing algorithms for generalized linear bandits with subexponential reward distributions. 

The paper is organized as follows: In Section 2, we introduce single-parameter NEFs and review key properties relevant to our analysis. Section 3 demonstrates the self-concordance property of subexponential (subgaussian) NEFs, with a quadratic (respectively, linear) growth rate of the stretch factor. Additionally, we establish the tightness of the linear growth rate for subgaussian NEFs. In Section 4 we apply these findings to subexponential GLBs and derive novel second-order regret bounds devoid of exponential dependencies on the problem parameter in the leading term. Proofs omitted from the main text are provided in the appendix, except for those of well-known results, which are referenced accordingly.

%% file: main/setting.tex
\section{Preliminaries}
In this section we first introduce the notation we will use. We then introduce 
natural exponential families, review some of their basic properties and illustrate the concepts introduced by means of an example. 
\subsection{Notation}\label{section:notation}
For a real-valued differentiable function $f$ defined over an open interval, we use $\dot f, \ddot f$ and $\dddot f$ to denote the first, second and third derivative of $f$. 
The set of reals is denoted by $\R$, the set of nonnegative reals by $\R_+$. 
For a set $\cU\subseteq \RR$, we denote its interior by $\cU^\circ$. 
For real numbers $a,b$, we use $a\wedge b$ and $a \vee b$ to denote the $\min\{a,b\}$ and $\max\{a,b\}$, respectively. 
With $\phi$ a logical expression,  $\II\{\phi\}=1$ if $\phi$ evaluates to true and $\II\{\phi\}=0$, otherwise.
We use $f\lesssim g$ to indicate that $g$ dominates $f$ up to a constant factor over their common domain. 
For $S\subseteq \R$, $x\in \R$, we let $S\pm x$ denote the set $\{ s\pm x\,:\, s\in S \}$. 
A distribution over the reals is centered if it has zero mean. We use $\Prob$ to denote the probability measure over the probability space that holds our random variables and we let $\E$ to denote the expectation corresponding to this measure. By $Y\sim Q$ we mean that the distribution of $Y$, a random variable, is $Q$.

\subsection{Single-parameter \NEF{}s}\label{section:NEF}
In this section we give our definitions for the \NEF. 
We only consider single-parameter natural exponential families when the base distribution is defined over the reals. We follow the approach of the beautifully written monograph of \citep{brown1986} that the reader is also referred to for any statements made about \NEF{}s with no proofs. 

Given a probability distribution $Q$
over $\RR$ let $M_Q:\RR\to \RR_+\cup\{+\infty\}$ denote its 
\emph{moment generating function} (\MGF):
\begin{align*}
M_Q(u)= \int \exp(uy)Q(dy)\,, \qquad u\in \R\,.
\end{align*}
We will find it convenient to also use the logarithm of the moment generating function, which is known as its \emph{cumulant generating function} (\CGF). We denote this by $\psi_Q:\RR \to \RR\cup\{+\infty\}$. Thus, $\psi_Q(u) = \log M_Q(u)$. 

Let $\cU_Q=\{u\in \RR: M_Q(u)<\infty\}$ denote the domain of $M_Q$  (and, thus the domain of $\psi_Q$).
As it is well-known, $\psi_Q$ is convex and hence $\cU_Q$ is always an interval (which, trivially, always contains $0$).
For a subset of $\cU_Q$, denoted by $\cU \subseteq \cU_Q$, we call $\cQ = (Q_u)_{u\in \cU}$ 
a natural exponential family (NEF) generated by $Q$
where for any $u\in \cU$ we have
\begin{equation*}
    Q_u(dy) = \frac{1}{M_Q(u)} \exp(uy)Q(dy)\,.
\end{equation*}
It follows that $Q_u$ is also a probability distribution over the reals for any $u\in \cU$ by definition.
An equivalent, useful form for $Q_u$ is $Q_u(dy) = \exp(uy - \psi_Q(u)) Q(dy)$.
In applications, $u$ denotes an unknown parameter that is to be estimated based on observations from $Q_u$. Thus, $\cU$ allows one to express extra restrictions on the admissible parameters beyond the limits imposed by $Q$. We call $\cU$ the \emph{parameter space}, and $\cU_Q$ the \emph{natural parameter space}.

The distributions $Q$, $Q_u$ and parameter $u$ are referred to as the \emph{base distribution}, 
the (exponentially) \emph{tilted (base) distribution} and the \emph{tilting parameter}. 
An \NEF is said to be \emph{regular} when $\cU_Q$ is open. 
It is easy to see that for any $u,u_0\in \cU_Q$, 
$Q_u = (Q_{u_0})_{u-u_0}$, where the distribution on the right-hand side stands for the tilt of $Q_{u_0}$ with parameter $u-u_0$. As such, up to a constant shift of the parameter space, in a regular \NEF, one can always assume that $0\in \cU_Q^\circ$. In fact, the same can be assumed for the parameter set, as long as $\cU^\circ$ is nonempty. If $\cU$ is an interval then $\cU^\circ=\emptyset$  means that $\cU$ is a singleton: An uninteresting case if we want to model a host of \emph{non-identical} distributions.


In a regular family, an equivalent way to parameterize a \NEF is using the mean function (cf. Theorem 3.6, page 74, of \cite{brown1986}), $\mu_Q: \cU_Q\to \R$, which is defined as 
\begin{equation*}
    \mu_Q(u)=
    \int y\,\, Q_u(dy)=\frac{\int y\exp(uy)Q(dy)}{M_Q(u)}.
\end{equation*}
Since $Q_0 = Q$, clearly, $\mu_Q(0)$ is just the mean of $Q$. To minimize clutter, when $Q$ is clear from the context, we will write $\mu$ instead of $\mu_Q$. To illustrate the developments so far, we consider the example when the base distribution is an exponential distribution.
\begin{example}[Exponential distributions]\label{ex:exp}
    For $\lambda>0$, let $Q$ be an exponential distribution with parameter $\lambda$: $Q(dy)=\II\{y\ge 0\} \lambda e^{-\lambda y} dy$. 
    As is well known, the \MGF of $Q$ takes the form
    $M_Q(u)=\frac{\lambda}{\lambda-u}$ when $u<\lambda$ and $M_Q(u)=\infty$ otherwise. 
    Thus, $\cU_Q=\{u\in \R:M_Q(u)<\infty\}=(-\infty,\lambda)$. The mean function takes the form of
    \begin{equation*}
        \mu(u)=\frac{\int_0^\infty \lambda y\exp(-\lambda y)\exp(uy)dy}{M_Q(u)}=\frac{1}{\lambda-u}, \qquad u<\lambda\,.
    \end{equation*}
\end{example}
In what follows, we will need the following proposition to relate the central moments of $Q_u$ to the derivatives of $\psi_q$, the CGF.
\begin{proposition}\label{prop:moments_of_NEF}
Let $\cU_Q^\circ$ be non-empty.
Then, $\psi_Q$ is infinitely differentiable on $\cU_Q^\circ$. 
Furthermore, the first three derivatives of $\psi_Q$ at $u\in \cU_Q$ give the first moment, second and third central moments of $Q_u$.
\end{proposition}
 In the context of \cref{ex:exp}, \cref{prop:moments_of_NEF} gives that $\psi_Q(u)= \log(\lambda)-\log(\lambda-u)$ when $u<\lambda$. Then,
$\dot\psi_Q(u) =\frac{1}{\lambda-u}$, agreeing with our earlier computation.

%% file: main/main_result.tex

\section{Self-concordance of NEFs}
This section contains the first set of main results of this paper. 
We start by giving our definition of self-concordance of NEFs, followed by a study of when this property is satisfied. We include a result that shows how self-concordance allows one to derive tail properties of members of the family from that of the base distribution.

In general, if the magnitude of a higher order derivative of a real function 
can be bounded in terms of a lower order derivative of the function, the function is said
to be self-concordant \citep{sun2019generalized}.
This property is useful for studying how fast the function changes with its argument, as well as for deriving useful bounds on how well the function can be approximated by low order polynomials
\citep{Nesterov1994InteriorpointPA,Nesterov04,sun2019generalized}.
\todoc{Strengthen this by adding relevant references.. }
In the context of single-parameter natural exponential families, we propose the following natural definition:
\begin{definition}[Self-concordant NEF]\label{defn:self-conc-NEF}
Let $\cQ = (Q_u)_{u\in \cU}$ be a NEF with parameter set $\cU\subset \cU_Q^\circ$ and some base distribution $Q$.
We say that $\cQ$ is \emph{self-concordant} if there exists a 
nonnegative valued function $\scfunc:\cU\to \R_+$ such that 
\begin{align}
|\ddot \mu(u) | \le \scfunc(u) \dot\mu(u) \qquad \text{for all } \quad u\in \cU\,.
\label{eq:scdef}
\end{align}
Any function $\scfunc:\cU\to \R_+$ that satisfies \cref{eq:scdef} is called a \emph{stretch function} of the NEF.
\end{definition}
This definition takes inspiration from the works of \citep{Bach2010self,sun2019generalized,faury2020improved} and \citep{janz2023exploration} who define an analogous property
(cf. Assumption~2 of \citep{janz2023exploration}). \todoc{is the \citep{faury2020improved} reference the correct one? please change if needed!}
By \cref{prop:moments_of_NEF}, we know that 
$\dot\mu(u)$ is the variance of $Q_u$, while $\ddot \mu(u)$ is the third central moment. 
Hence, $\dot\mu(u)$ is nonnegative, which explains why there is no absolute value on the right-hand side
of \cref{eq:scdef}.

According to our definition, we require that the (absolute value) of the second derivative of $\mu$ is bounded by the first derivative, up to the ``stretch factor'' $\scfunc(u)$.
Clearly, provided that $\dot\mu$ is positive over $\cU$, self-concordance is equivalent to stating that 
$\selfconc_Q(u)\doteq\frac{|\ddot\mu(u)|}{\dot\mu(u)}$ is finite valued over $\cU$. \todoc{at one point we should give an example of a non self-concordant NEF. 
Given our theorem, for this to happen, we need to have $0\in \partial U_Q$. But, say, for inverse Gaussian, where this happens  and in particular $U_Q = (-\infty,0]$ (a steep, nonregular family), $\scfunc_Q(u) = \frac{c}{-u}$ for $c>0$, $u<0$. So the behavior is the same as for $\mathrm{Exp}(1)$.
}
When $\dot\mu(u)=0$ for some $u\in \cU$, one can show that $Q_u$ must be a Dirac distribution and hence so does $Q$ ($Q$ and $Q_u$ share their support). In this case, we also have $\ddot \mu\equiv 0$ and hence any nonnegative valued function is a valid stretch-function. In particular, $\selfconc_Q\equiv 0$ is also a valid stretch function. 

It turns out that studying self-concordance property of distributions in detail can turn out much finer bounds than naively bounding $\ddot \mu$ and $\dot \mu$ separately.
\begin{example}[Avoiding exponential dependencies]
    Consider a NEF $\cQ$ with base distribution  $Q$, a Bernoulli distribution with parameter $1/2$, we have $\cU_Q=\R$, $\mu(u)=\frac{1}{1+e^{-u}}$ and 
    thus $\dot\mu = \mu(1-\mu)$, $\ddot \mu=\dot\mu(1-2\mu)$. Hence
    $\cQ$ is $\scfunc_Q$-self-concordant with $\scfunc_Q(u)\le 1$ for all $u\in \RR$.
    This is an example where naively bounding $\scfunc_Q$, 
    by bounding the numerator and denominator separately 
    over $[-s,s]$ for $s>0$ 
    gives a quantity of size $e^s$, which lags far behind the constant we obtained with a direct calculation,
    or what we can get from the result in \cref{sec:sg}, which show a scaling of order $O(s)$.
\end{example}

As opposed to earlier literature where self-concordance is used \citep{Nesterov1994InteriorpointPA,Bach2010self,sun2019generalized}, we allow a 
non-constant stretch-factor $\scfunc$. As it turns out, this is necessary if $\cU = \cU_Q^\circ$ is to be allowed:

\begin{example}[Non-constant stretch factor]\label{ex:sc}
    Consider a NEF $\cQ$ with base distribution $Q$.
    For $Q$ an exponential distribution with parameter $\lambda>0$, 
    $\cQ$ is self-concordant over $\cU = \cU_Q = (-\infty,\lambda)$ with $\scfunc_Q(u)=\frac{2}{\lambda - u}$, $u<\lambda$.
    Indeed, a simple calculation gives that for $u<\lambda$,
    $\dot\mu(u)=(\lambda-u)^{-2}$ and $\ddot\mu(u) = 2(\lambda-u)^{-3}$, and so $|\ddot\mu(u)|/\dot\mu(u) = 2/(\lambda-u)$.

\end{example}
As was shown above, for the NEF built on the exponential distribution with parameter $\lambda$, there is no constant stretch factor that makes the NEF self-concordant over the entirety of the natural parameter space. 
The main result of the next section shows that a non-constant stretch factor with growth similar to the exponential always exists provided that the base distribution is \emph{subexponential}.%
\footnote{
Here, we follows the terminology used in the concentration of measure literature where these distributions are also knowns as subgamma distributions
\citep{BoLuMa13,vershynin2018high,wainwright2019high}, or light-tailed distributions.
This is to be contrasted to the use of the same term in the theory of heavy-tailed distributions, where subexponential refers to a much larger class of distributions \citep{Goldie1998}.}
As it turns out, the growth of the stretch factor plays an important role in applications: 
Among other things, it allows us to conclude that the tilted distributions are subexponential,
 as captured by the next result, which is a slightly generalized version of an analogous result of \citep{janz2023exploration}: 
\begin{restatable}[From self-concordance to light tails]{lemma}{SCtoTail}\label{lem:tiltsubexp} 
Let $\cQ = (Q_u)_{u\in \cU}$ be a NEF which is self-concordant with stretch function $\selfconc:\cU \to \R_+$ where $\cU$ is a subinterval of $\cU_Q^\circ=(a,b)$. 
Then, for any $u\in \cU$, 
    \begin{align}
    \label{eq:sctail}
        \psi_{Q_u}(s)\le s\mu(u)+s^2\dot\mu(u) \spaced{for all}s\in [-\log(2) / K, \log(2)/K]
        \cap (a-\inf \cU, b-\sup \cU)\,,
    \end{align}
    where $K = \sup_{u\in \cU} \selfconc(u)$.
\end{restatable} 
Note that when $\cU$ is a strict subset of $\cU_Q^\circ$, $0\in (a-\inf \cU, b-\sup \cU)$ 
and thus the result is nontrivial as long as $K<\infty$.
To interpret this result,
recall that a centered distribution $Q$ is called \emph{subexponential} with nonnegative parameters $(\nu,\alpha)$ if 
\begin{align*}
\psi_Q(s)\le s^2 \nu^2/2\,, \qquad \text{for all } |s|\le 1/\alpha
\end{align*}
 (cf. Definition~2.7 \citep{wainwright2019high}). 
A consequence of this is that the mean of $n$ independent random variables drawn from $Q$, with high probability, will be in an zero-centered interval of length $O(\sqrt{\nu^2/n} + \frac{1}{\alpha n})$. Here, the first term describes a ``subgaussian'' behavior, while the second a ``pure subexponential'' behavior. 
Assume for simplicity that $a=-b$ and $\inf \cU = -\sup\cU$.
From \cref{eq:sctail} it follows that $Q_u$, when centered, is subexponential with parameters $\nu^2 =2\dot\mu(u)$, twice the variance of $Q_u$, and $\alpha = \min(\log(2)/K,b-\sup \cU)$. In particular, we see that the growth of $\selfconc$ impacts the lower-order term in the confidence interval (the term $1/(\alpha n)$), but not the leading term, which is governed by the variance of $Q_u$.
It follows that understanding how fast $\selfconc$ can grow as its argument approaches the boundary of $\cU_Q$ is thus important, although its impact only appears in a low-order term. 

We now turn to our first main result that shows that the growth of $\selfconc$ is at most inverse polynomial
provided that the base distribution itself is subexponential.
\todoc{Discussion points:
Is statistical estimation even feasible when $\cU = \cU_Q^\circ$? Consider, for example, the case
when $Q$ is the exponential distribution with parameter $\lambda$.
Note that in this case $Q_u$ is the exponential distribution with parameter $\lambda-u$ (defined only for $u<\lambda)$.
Thus, we are essentially asking for the efficiency of estimating the parameter of the exponential distribution with a near zero positive parameter $\lambda$.
Actually, we care about the mean of this.
Which can be estimated using the sample mean.
The variance is $1/\lambda^2$. So as $\lambda\to 0+$, we have worse estimates.
In particular, the confidence interval width blows up with $1/\lambda$ as $\lambda \to 0+$.
So one should be able to get confidence sets still. What happens with bandits?
The arm with the highest mean reward has the largest variance. 
The means also grow unboundedly, so regret may grow unboundedly,
if the parameter space is not restricted.}


\subsection{Self-concordance with a subexponential base}
We start with recalling equivalent definitions of subexponential distributions, which will be useful to interpret our results.
In particular, 
according to Theorem~2.13 of \citep{wainwright2019high}, given a zero-mean distribution $Q$ over the reals, the following are equivalent:
\begin{enumerate}[(i)]
\item $Q$ is subexponential;
\item The MGF of $Q$ is defined in an open neighborhood of zero;
\item For some positive constants $C,c>0$, $\Prob{ |Y|\ge t } \le C \exp(-c t)$ where $Y\sim Q$.
\end{enumerate}
It will be useful to have a quantitative version of this statement. For this, we separate the left and the right tails (as in some application they have different behaviors, which we may care about).
The quantitative result for the right-tail is as follows (the result for the left-tail is omitted; it follows by symmetry):
\begin{restatable}{proposition}{SubExpEq}\label{prop:subexpeq}
Let $c_1,C_1,c>0$, $Y\sim Q$ and assume that $\EE{Y}=0$. 
If 
\begin{align}
\Prob{Y\ge t} \le C_1 \exp(-c_1 t) \qquad \text{for all } t\ge 0 \label{eq:righttail}
\end{align}
then
for any $0\le \lambda < c_1$, $M_Q(\lambda)< 1+ \frac{C_1 \lambda^2}{c_1(c_1-\lambda)}$.
Furthermore, for any $c> 0$ such that $M_Q(c)<\infty$,
$\Prob{Y\ge t} \le M_Q(c) e^{-c t}$ holds for all $t\ge 0$.
\end{restatable}
The first part is nontrivial; the second follows easily from Chernoff's method. The proof is given in \cref{sec:subexp}.

Let 
\[
\ER(c_1,C_1) = \{ Q \,: \,  Y=X-\EE{X} \text{ satisfies  \cref{eq:righttail} where } X\sim Q \}\,.
\]
In words, $\ER$ is the class of distributions over the reals whose right-tail displays an exponential decay governed with the rate parameter $c_1>0$ and scaling constant $C_1>0$.
Similarly, we let $\EL(c_1,C_1)$ be the class of distributions $Q$ over the reals such that for $X\sim Q$, $Y=\EE{X}-X$ satisfies \cref{eq:righttail}.
With this notation, the first part of the previous proposition is equivalent to that if $Q\in \ER(c_1,C_1)$ is centered then $M_Q$ stays below the function $\lambda \mapsto 1 +\frac{C_1 \lambda^2}{c_1(c_1-\lambda)}$ over the interval $[0,c_1)$. In particular, this means that $\cU_Q$ contains $[0,c_1)$.

With this, we are ready  to present our theorem that establishes the self-concordance property of NEFs with subexponential tails.  
\begin{restatable}{theorem}{ScThm}
\label{thm:informal-self-conc}
Let $Q\in \ER(c_1,C_1)\cap \EL(c_2,C_2)$ for some positive constants $c_i,C_i$, $i=1,2$.
Then, the NEF $\cQ = (Q_u)_{u\in \cU}$ is self-concordant.
Moreover,  
the function $\selfconc:\cU \to \R_+$ defined by 
    \begin{align*}
            \scfunc(u)=
        \begin{cases}
            \frac{3}{2}\left[2eC_1c_1\left(\frac{1}{c_1-u}\right)^2+\frac{ub}{c_1-u}\right]+ G_Q(C_1,C_2,c_1,c_2)& \text{if } 0\le u <c_1 \\
            \frac{3}{2}\left[2ec_2C_2\left(\frac{1}{c_2+u}\right)^2+\frac{-u b}{c_2+u}\right]+G_Q(C_2,C_1,c_2,c_1)&\text{if } -c_2<u<0,
        \end{cases}
        \end{align*}
       is a stretch function for the NEF $\cQ$, where $G_Q(M_1,M_2,m_1,m_2)$ is a
	polynomial whose coefficients depend on $Q$. 
\end{restatable}
The exact expression of $G_Q$ can be found in \cref{eq:G_Q}.
Let $\cQ$ be a NEF with base distribution $Q$ when $Q$ is a zero-mean Laplace distribution with variance $2\lambda^2$. In this case, we can choose $c_1=c_2=1/\lambda$, $C_1=C_2=1/2$. The actual stretch function is $\Gamma_Q(u)=\frac{8\lambda^2 |u|}{(1/\lambda- u)(1/\lambda+u)(3\lambda^4 u^2+\lambda^2)}$. Thus, the above theorem gives the correct behavior in that both the stretch function from the theorem and the actual stretch function $\scfunc_Q$ blow up with the inverse of the distance to the boundaries of $\cU_Q$, except that the actual growth scales linearly with the inverse distance, while the theorem gives a quadratic scaling. It remains an open problem to see whether this quadratic order can be improved.

The following corollary is an immediate result of \cref{lem:tiltsubexp} and \cref{thm:informal-self-conc} (\cref{section:appendix_proof_regular_NEF_subexp}),
but can also be proved directly from the definitions (and we include a direct proof as part of the proof of 
\cref{thm:informal-self-conc}). \todoc{add to the appendix the direct proof, which should just point to the two lemmas that prove this.}
\begin{corollary}[Distributions in regular NEFs are subexponential]\label{coro:regular_NEF_subexp}
Let $u\in \cU_Q^\circ$. Then $Q_u$ is subexponential both on left and right.
\end{corollary}
Because of this result, there is essentially no loss in generality in only considering the subexponential case when working with NEFs. In particular, self-concordance in NEFs is ``almost free''.

\paragraph{Proof sketch for \cref{thm:informal-self-conc}}\label{section:proof_sketch_se} \todoc{Need to be checked by me}We sketch here the result for $\mu(0)=0$, $\dot\mu(0)>0$ and $u\geq 0$. The arguments used to extend the result to all cases can be found in \cref{section:self-conc-appendix}. By \cref{prop:moments_of_NEF},  bounding the stretch function of a NEF amounts to showing
\begin{equation*}
    \selfconc_Q(u):=\frac{\int |(y-\mu(u))|^3Q_u(dy)}{\int (y-\mu(u))^2Q_u(dy)}\le \scfunc(u),
\end{equation*}
where the division by $\var(Q_u)=\int (y-\mu(u))^2Q(dy)$ is justified by \cref{lem:mu_positive}.  We split the proof of that upper bound into two steps: controlling the variance and the absolute third moment.
\paragraph{Step 1: Controlling the variance}
Since $\dot\mu(0) = \var(Q)>0$, there exists a $Q$-dependent constant $\eta>0$ and an interval  $[-b,-a]\subset \mathbb{R}_{<0}$ s.t. $Q([-b,-a])\ge \eta$ (\cref{lem:lb_probability_base_measure}). With this observation, we can show (\cref{lem:LBdenominator}): 
    \begin{equation}
     \dot\mu(u) \ge a^2\eta\frac{e^{-ub}}{M_Q(u)}.\label{eq:LB-denominatorsketch}
    \end{equation}
Thus, the second moment decreases at most exponentially with the parameter $u$. 

\paragraph{Step 2: Controlling the absolute third central moment} First, we use a classical result on moments of random variables (see the proof of $i\Rightarrow ii$ for prop.2.5.2 in \cite{vershynin2018high}):
\begin{align}
    \int |y-\mu(u)|^3Q_u(dy) &= \int_0^B3t^2\PP(|Y-\mu(u)|\ge t)dt + \int_B^\infty 3t^2\PP(|Y-\mu(u)|\ge t)dt\nonumber\\
    &\le \frac{3}{2}B\dot\mu(u) + \int_B^\infty 3t^2\PP(|Y-\mu(u)|\ge t)dt\label{eq:int_decomp_2},
\end{align}
where $Y\sim Q_u$
and $B$ is a constant to be optimized. It should remain small enough for the first term not to blow up. On the other hand, it should be large enough for the second term divided by the lower bound obtained on $\dot \mu (u)$ to remain controlled.

To upper bound the second term in \cref{eq:int_decomp_2}, we start by showing that the right tail of the tilted distribution $Q_u$ is also subexponential (\cref{lem:UBuppertailQuexp}):
    \begin{equation}\label{eq:UB-uppertailQuexp}
        Q_u((t,\infty))\lesssim \frac{1}{M_Q(u)}e^{-(c_1-u)t}\frac{1}{c_1-u}\quad \text{for} \quad 0\le u< c_1\,.
    \end{equation}
Following this lemma, we get an upper bound on $\mu(u)$ (\cref{lem:UBmuuexp}):
    \begin{equation}\label{eq:UB-muuexp}
        0\le \mu(u)\lesssim \left(\frac{1}{c_1-u}\right)^2 \quad \text{for} \quad 0\le u< c_1\,.
    \end{equation}
In the proof, we bound separately the positive and negative values in the second term of \cref{eq:int_decomp_2}:
\begin{align*}
    \int_B^\infty 3t^2\PP(|Y-\mu(u)|\ge t)dt&=\underbrace{\int_B^\infty 3t^2\PP(Y\ge \mu(u)+t)dt}_{\spadesuit} + \underbrace{\int_B^\infty 3t^2\PP(Y\le \mu(u)-t)dt}_{\heartsuit}.
\end{align*}
We give here the sketch of proof on the bound of $\spadesuit$ as the proof for bounding $\heartsuit$ is nearly identical. From plugging in \cref{eq:UB-uppertailQuexp}, we get:
\begin{align*}
    \spadesuit&\lesssim  \frac{1}{M_Q(u)}\frac{1}{c_1-u}\int_{B}^\infty 3t^2e^{-(c_1-u)(t+\mu(u))}dt.
\end{align*}

By choosing $B\gtrsim \frac{ub}{c_1-u}+\left(\frac{1}{c_1-u}\right)^2$, some algebra gives:
\begin{align*}
    \spadesuit
    &\lesssim \frac{e^{-ub}}{M_Q(u)}\left(\frac{1+u^2b^2}{c_1^3 C_1^3}\right).
\end{align*}

Setting $B\gtrsim\left(\frac{1}{c_1-u}\right)^2+\frac{ub}{c_2+u}$ gives a similar bound on $\heartsuit$. Chaining the bounds on $\spadesuit$ and $\heartsuit$ with \cref{eq:int_decomp_2} and \cref{lem:LBdenominator} finishes the proof. \hfill \( \Box\)

\subsection{Self-concordance with a subgaussian base}\label{sec:sg}
In this section we refine the previous result for NEFs by considering the case when the base distribution is subgaussian. \todoc{Probably after the deadline: We can shift the domain by tilting the base. Hence, if $0\in \cU_Q^{\circ}$ there are really $4$ cases, based on whether $\sup\cU_Q$ and $\inf \cU_Q$ are finite or not. The previous section considered, finite both, finite right, infinite left, this considers infinite both. How about the case that is left out? And could the analysis of the single sided finite case be improved by doing the calculations like in this section for the infinite side? Or that is already done?}
 Let $\sigma>0$. Recall that a centered distribution $Q$ is $\sigma$-subgaussian if
for all $u\in \R$,
$M_Q(u)\le e^{\sigma^2 u^2/2}$ (or, equivalently, $\psi_Q(u) \le \sigma^2 u^2/2$ for any $u\in \R$).
A non-centered distribution is $\sigma$-subgaussian, if it is subgaussian after centering.
Similarly to the subexponential case, one can show that a centered distribution $Q$ is subgaussian if and only if for some $\tau,C>0$, it holds that for any $t\ge 0$, $\Prob{|X|\ge t } \le C \exp(- t^2/(2\tau^2) )$ 
where $X\sim Q$ (cf. Proposition 2.5.2 of \cite{vershynin2018high}).\footnote{As for the quantitative relation between the parameters, it can be shown (\cite{rigollet2023highdimensional}) that if
for all $u\in \R$,
$M_Q(u)\le e^{\sigma^2 u^2/2}$, then $\Prob{|X|\ge t } \le 2 \exp(- t^2/(2\sigma^2)$, and if $\Prob{|X|\ge t } \le 2 \exp(- t^2/(2\sigma^2))$, then for all $u\in \R$,
$M_Q(u)\le e^{4\sigma^2 u^2}$.} 

Our promised result is as follows:
\begin{restatable}{theorem}{ScgUBThm}\label{thm:informal-self-conc-sg-ub}
Let $Q$ be subgaussian. Then, the NEF $\cQ = (Q_u)_{u\in \RR}$ is self-concordant 
and $\selfconc_Q(u) = O(|u|)$, $u\in\R$.
\end{restatable}
As it turns out, the linear growth exhibited in the previous result is tight for NEFs
with a subgaussian base distribution: 
\begin{restatable}{theorem}{ScgLBThm}\label{thm:informal-self-conc-lb}
    There exists a distribution $Q$ that is subgaussian such that  
    $\limsup_{u\to \infty} \selfconc_Q(u)/ u >0$.
\end{restatable}
Again, this shows that even if we stay with subgaussian distributions, it would be limiting to only consider NEFs that are self-concordant with a bounded (or constant) stretch function over their natural parameter space.

%% file: main/OFU-GLB.tex
\section{Generalized Linear Bandits}
In this section we apply the self-concordance property of subexponential NEFs in order to derive novel confidence sets and regret bounds for subexponential generalized linear bandits. To our knowledge, these are the first such results for parametric bandits with subexponential rewards. 

\subsection{Bandit model}\label{section:bandit_model}
Following \citet{Filippi2010parametric}, we consider
 stochastic generalized linear bandit (GLB) models
  $\cG$ specified by a tuple $(\cX,\Theta,\cQ)$, where
$\cX\subseteq \R^d$ is a non-empty arm set, 
$\Theta \subseteq \R^d$ is a non-empty set of potential parameters, both closed for convenience, 
$\cQ=(Q_u)_{u\in \cU_Q}$ is a NEF with base distribution $Q$. \todoc{An alternative to this is to 
assume only $\cX$ and 
$\cQ=(Q_u)_{u\in \cU}$ and then set
$\Theta= \{ \theta\in \R^d\,: \theta^\top x \subseteq \cU \text{ for all } x\in \cX \}$.
What are the pros and cons? After deadline.
}
Without the loss of generality we assume that $\cX$ is a compact subset of the Euclidean unit ball of $\R^d$. \todoc{note that $\cG$ is actually a generalized linear model.. and a special case of that.. (nothing in $\cG$ is ``bandit specific''). After deadline.}

In each round $t\in \NN^+$, the learner selects and plays an arm $X_t\in \cX$. 
As a response, they
receive a reward $Y_t$ sampled from 
the distribution $Q_{X_t^\top\theta_\star}$, where $\theta_\star\in \R^d$ is a parameter of the bandit environment, which is initially unknown to the learner. The learner's goal is to maximize its total expected reward. The GLB is \emph{well-posed} when 
\[
\cU\doteq \{ x^\top \theta \,: \, x\in \cX, \theta \in \Theta \}\subseteq \cU_Q^\circ
\] 
holds, which we assume from now on. 
The condition that $\cU \subseteq \cU_Q$ simply ensures that the reward distributions $Q_{x^\top \theta}$ are defined regardless of the value of $(x,\theta)\in \cX\times \Theta$.
We require the stronger condition $\cU \subseteq \cU_Q^\circ$ to exclude the boundaries of the interval $\cU_Q$. This way we avoid pathologies that arise when a parameter reaches the boundary of $\cU_Q$ (e.g., when $Q$ is the exponential distribution with parameter $\lambda$, 
the mean and variance of $Q_u$ grow unbounded as $u$ approaches $\lambda$ from below).
\todoc{I guess there is a stronger argument that shows that somehow this must be done? or not?\\SL: Just a note here: we discussed the arms with parameter $1-\varepsilon$ and $1-\varepsilon-\varepsilon^2$.}
%

The expected reward in round $t$ given that the learner plays $X_t$ is 
$
    \EE[Y_t|X_t] = \mu(X_t^\top\theta_\star).
$
%
%
%
%
The performance of the learner will be assessed by their pseudo regret $R(T)$, which is 
the total cumulative 
shortfall of the mean reward of the arms the learner chose relative to optimal choice:
\begin{equation*}
    R(T)=\sum_{t=1}^T \mu(x_\star^\top \theta_\star) - \mu(X_t^\top \theta_\star)\,.
\end{equation*}
Here,
 $x_\star\in \argmax_{x\in \cX}\mu(x^\top\theta_\star)$ is the arm that results in the best possible expected reward in a round. For simplicity, we assume that such an arm exists. We establish guarantees of our algorithm for a subclass of GLBs, captured by the following assumption:
\begin{assumption}[Subexponential base]\label{ass:model}
The base distribution $Q$ is subexponential both on the left and the right. In particular,
$Q\in \ER(c_1,C_1)\cap \EL(c_2,C_2)$ for some $c_i,C_i$ ($i=1,2$) positive numbers. 
Furthermore, $-c_2<\inf \cU\le \sup\cU <c_1$ and $c_1,c_2$ are known to the learner.
%
\end{assumption}
Note that in a well-posed GLB, $\inf \cU_Q <\inf \cU \le \sup \cU < \sup\cU_Q$.
In light of this, the assumption just stated boils down to whether $0\in \cU_Q^\circ$, which, as discussed, is free when $\cU_Q^\circ\ne \emptyset$. 
Indeed, when $0\in \cU_Q^\circ$ holds, $\inf \cU_Q <0 < \sup \cU_Q$ and one can always find positive values $c_1,c_2$
such that $\inf \cU_Q < -c_2 < \inf \cU \le \sup \cU < c_1 < \sup \cU_Q$.
Then, from \cref{prop:subexpeq} and some extra calculation one can conclude that 
the $Q\in \ER(c_1, e^{-c_1 \EE{X}} M_Q(c_1)) \cap \EL( c_2, e^{-c_2 \EE{X}} M_Q(-c_2))$. 
Since it is assumed that the learner has access to $Q$, we see that \cref{ass:model} can be satisfied 
whenever $0 \in \cU_Q^\circ$, which we think is a rather mild assumption. We will also for the sake of simplicity assume that the learner has access to $S_0=\sup \{ \|\theta\|\,:\, \theta\in \Theta \}$,
$S_2=\inf \cU$, $S_1 = \sup \cU$. These values will be used in setting the parameters of the algorithm. Note that it is not critical that the learner knows these exact values; appropriate bounds suffice. We also assume that the learner is given access to an upper bound on the worst-case variance over the parameter space $\cU$:
\begin{assumption}[Bounded Variance]
    \label{ass:bdd_var} 
    The learner is given $L\ge 1$ such that 
    $\sup_{u\in \cU} \dot\mu(u) \leq L$. 
\end{assumption}
Note that since the GLB is well-posed, $\sup_{u\in \cU} \dot\mu(u)<\infty$ is automatically satisfied.
Also, there  is no loss of generality in assuming $L\ge 1$.
A crude upper bound on $\sup_{u\in \cU} \dot\mu(u)$ is
$C_1 e /(c_1-S_1)^3\vee C_2 e/(c_2+S_2)^3$, \todoc{add result to appendix. low low low priority or after deadline.}
so this assumption is implied by the previous one. 
    
\subsection{The OFU-GLB algorithm and its regret}
Just like the previous works \citep{Filippi2010parametric,pmlr-v70-li17c,faury2020improved} 
which considered special cases of the generalized linear bandit problem, our algorithm follows the ``optimism in the face of uncertainty'' principle. In each time step,
the algorithm constructs a confidence set $\cC_t$, based on past information, that contains the unknown parameter $\theta_\star$ with a controlled probability. Next, the algorithm chooses a parameter $\theta_t$, in the confidence set $\cC_t$, and an underlying action $X_t\in \cX$ such that the mean reward underlying $X_t$ and $\theta_t$ is as large as plausibly possible. 
Since $\mu=\mu_Q$ is guaranteed to be an increasing function (recall that $\dot \mu(u)$ is the variance of $Q_u$ and is hence nonnegative), it suffices to find the maximizer of $x^\top \theta$ where $(x,\theta)\in \cX\times \Theta$. We call our algorithm,  shown in \cref{alg:OFU-GLB}, OFU-GLB (Optimism in the Face of Uncertainty in Generalized Linear Bandits).
The main novelty here is that our bandit model makes minimal assumptions.

\input{main/alg}

The confidence set is based on ideas from the work of \citet{janz2023exploration}.
Note that this paper analyzed a randomized method for those GLBs where $\cU_Q = \R$ and 
$\sup_{u\in \cU} \selfconc_Q (u) <\infty$.
The assumption that $\cU_Q=\R$ is restrictive, as it does not allow many common distributions (e.g., the exponential distribution).
Thanks to \cref{thm:informal-self-conc}, under our assumptions, 
$\sup_{u\in \cU} \selfconc_Q (u) <\infty$ follows. Then, 
an appropriate confidence set can be constructed based on 
\cref{lem:tiltsubexp}, which also extended the corresponding result of \citet{janz2023exploration}.
While the confidence set construction is based on the ideas of \citet{janz2023exploration},
the main steps of the analysis are taken from \cite{Abeille2020InstanceWiseMA} who analyzed logistic bandits, which are $1$-self-concordant.
Our result follows by carefully modifying the proof of \cite{janz2023exploration} and carefully propagating both the effect of replacing their confidence set with a different one, and the effect of $\sup_{u\in \cU} \selfconc_Q (u) >1$. This leads to the main result on GLBs: \todoc{Note if we have $Q$ Dirac then all algorithms have zero regret. In the other case $\kappa$ is finite.}

\begin{theorem}[Regret upper bound of OFU-GLB]\label{thm:informal-regret-bound}
    Let $\delta\in (0,1]$ and $T$ a positive integer and consider a well-posed 
    GLB model $\cG = (\cX,\Theta,\cQ)$ 
    and assume that \cref{ass:bdd_var,ass:model} hold.
    For $\theta_\star \in \Theta$, 
    let $\kappa(\theta_\star) = \frac{1}{\dot\mu(x_\star^\top\theta_\star)}$
	 and let $\mathrm{Regret}(T,\theta_\star)$  stand for the $T$-round regret of OFU-GLB
	 when it interacts with a GLB specified by $\theta_\star$.
    Then, with an appropriate construction of $\cC_t$, 
     for any $\theta_\star\in \Theta$, it holds that 
    with probability at least $1-\delta$, \todoc{where is the $\delta$ dependence?}
    \begin{equation*}
        \mathrm{Regret}(T)= \tilde {\mathcal O}\left(d\sqrt {\dot\mu(x_\star^\top\theta_\star)T}+d\kappa(\theta_\star)\right)\,,
    \end{equation*}
    where $\tilde O(\cdot)$ hides polylogarithmic factors in $T, d, L, 1/\delta$ and constants
    that depend on the base distribution $Q$. \todoc{add observation that the regret bound separates $Q$ dependent from $\theta_\star$ dependent constants.}
\end{theorem}


An essential quality of the result is that it makes the dependence of the regret on the instance $\theta_\star$ explicit.
Recalling that in a NEF, $\dot\mu(u)$ is the variance of the tilted distribution $Q_u$, we see
that the leading term (shown as the first term on the right-hand side of the last display) scales with the variance of the optimal arm's reward distribution. In queuing theory, the service times of agents (actions) in an environment are often modeled as exponentially distributed random variables \citep{queue_times}. When aiming to minimize service times (maximize negative reward) with an exponential bandit model (with mean function $\mu(x) = -1/x$), the variance of the optimal arm's service time lower bounds that of the other arms. 
Furthermore, the dependence on $\kappa$, a term that is inversely proportional to the optimal variance, is pushed to a second order term. In logistic bandits, $\kappa$ can be exponentially large in the size of the parameter set $S_0$ and thus much attention has been focused on mitigating its effect \citep{faury2020improved,jun2021improved,janz2023exploration}. Our regret bound also matches the lower bound in logistic bandits given by \cite{Abeille2020InstanceWiseMA}, thus our analysis is tight for this special case.

%% file: main/alg.tex
\begin{algorithm*}[tbh]
    \caption{The OFU-GLB Algorithm}
    \label{alg:OFU-GLB}
    \begin{algorithmic}
    	\REQUIRE GLB instance $\cG =(\cX,\Theta,\cQ)$
		\FOR{$t= 1,2,\dots$}
			\STATE Construct $\cC_t\subset \Theta$ based on $((X_s,Y_s))_{s<t}$ and $\cG$
			\STATE  Compute $(X_t, \theta_t) \in \argmax_{(x, \theta) \in  \cX \times \cC_t} x^\top\theta$		
			\STATE  Select arm $X_t$ and receive reward $Y_t\sim Q_{X_t^\top \theta_\star}$\label{alg:arm}
		\ENDFOR
	\end{algorithmic}
\end{algorithm*} 

%% file: main/conclusion.tex
\section{Conclusions and Future Work}
The main contribution of this work establishes that all subexponential NEFs are self-concordant with a polynomial-sized stretch factor. We then applied this finding and derived regret bounds for subexponential GLBs that scale with the variance of the optimal arm's reward, which is the smallest  variance amongst all arm's rewards in problems such as: minimizing service times \citep{queue_times} or minimizing insurance claim severity (dollars lost per claim) \citep{goldburd2016generalized}. 

Our findings also have implications when performing maximum likelihood estimation with subexponential NEFs, which includes a rich family of generalized linear models (GLMs). Since the log loss in a NEF is the sum of a linear function and the NEF's CGF,  
the GLM's loss is self-concordant in the sense of (say) \cite{bach2014adaptivity} whenever the NEF is self-concordant. 
While this is outside of the scope of our paper, it follows that this family of GLMs enjoy: $(i)$ fast rates of convergence to the minimizer for regularized empirical risk minimization \cite{pmlr-v99-marteau-ferey19a}, $(ii)$ fast rates for averaged stochastic gradient descent \cite{bach2014adaptivity} and $(iii)$ fast rates for constrained optimization with first-order methods \cite{dvurechensky2020self}, without restrictive conditions on bounded responses, which previous works had to assume to achieve these results.

One interesting direction for future work would be either deriving a matching lower bound on the stretch function for subexponential NEFs or tighter analysis that matches the lower bound for subgaussian NEFs. 
Another potential avenue for future work would be in extending our results to other exponential families, beyond NEFs. 

%% file: appendix/self_concordance_thm.tex

\section{Extra notation}
The following extra notation will be used in the appendix:
For vector $x \in \mbR^d$, we let $\|x\|$ denote its $\ell_2$-norm and for positive definite matrix $M \in \RR^{d\times d}$, we use $\|x\|_M = \sqrt{x^\top M x}$ to denote its $M$-weighted $\ell_2$-norm. 

\section{On subexponential (or ``light tailed'') distributions}
\label{sec:subexp}
We first prove \cref{prop:subexpeq}, which we repeat for the convenience of the reader:
\SubExpEq*
We follow the proof of  Theorem~2.13 from the book of \citet{wainwright2019high}.
\begin{proof} We start with the second part. For this let $t\ge 0$, $c>0$. Then, by Chernoff's method, we have
\[
\Prob{Y\ge t} \le \EE[e^{cX}] e^{-c t}= M_Q(c) e^{-c t}\,.
\]

The first part requires more work. Let us start by bounding the $p$-th moment of the positive part of $Y$, which we denote by $Y_+$ (hence, $Y_+ =\max(Y,0)$). We have
\begin{align*}
\mathbb{E}[Y_+^p] 
& =\int_0^{\infty} \Prob{Y_+^p \geq u} d u \\
& =p \int_0^{\infty} \Prob{Y_+ \geq t} t^{p-1} d t \tag{change of variables with $u=t^p$} \\
& =p \int_0^{\infty} \Prob{Y \geq t}  t^{p-1} d t \tag{for $t>0$, $\{Y_+\ge t\}=\{ Y\ge t \}$}\\
& \leq C_1 p \int_0^{\infty} e^{-c_1t} t^{p-1} d t \tag{assumption on $Y$}\\
& \leq \frac{C_1\,p}{c_1^p}\int_0^{\infty}  e^{-u} u^{p-1} d u \tag{change of variables with $u=c_1 t$}\\
& =  \frac{C_1\,p}{c_1^p} \Gamma (p-1) \tag{definition of the $\Gamma$ function}\\
& =\frac{C_1}{c_1^p} p! \tag{property of the $\Gamma$ function}
\end{align*}
Now let $0\le \lambda<c_1$. 
Since $Y\le Y_+$, we have
$M_Q(\lambda)=\E[e^{\lambda Y}]\le \E[e^{\lambda Y_+}]$. Hence,
by the power-series expansion of the exponential, we get
\begin{align*}
M_Q(\lambda)& \le \mathbb{E}\left[e^{\lambda Y_+}\right]\\
 & =1+\sum_{p=2}^{\infty} \lambda^p \frac{\mathbb{E}\left[Y_+^p\right]}{p!} \\
& \leq 1+C_1\sum_{p=2}^{\infty} \left(\frac{\lambda}{c_1}\right)^p \\
& \leq 1+C_1\frac{\lambda}{c_1}\frac{\lambda}{c_1-\lambda}\,.\qedhere
\end{align*}
\end{proof}
We note in passing
that since $\lambda<c_1$, 
$1+C_1\frac{\lambda}{c_1}\frac{\lambda}{c_1-\lambda}\le 1+C_1\frac{\lambda}{c_1-\lambda}$.
That \cref{eq:righttail} implies that $M_Q(\lambda)\le 1+C_1\frac{\lambda}{c_1-\lambda}$
can also be obtained by refining the proof of Theorem~2.13 from the book of \citet{wainwright2019high}. 

\section{Proof of \cref{thm:informal-self-conc}}
\label{section:self-conc-appendix}

For the convenience of the reader we restate the theorem to be proven:
\ScThm*

Note  that the function $\scfunc$ as defined above
is non-decreasing on $\cU\cap \RR^+$ and non-increasing on $\cU \cap \RR^-$.

The actual form of $G_Q$ is as follows: let $a,b,\eta>0$ be such that $Q([-b+\mu(0),-a+\mu(0)])>\eta$ and $-a<0$. Then, 
\begin{align}
        G_Q(M_1,M_2,m_1,m_2)&=\frac{3}{2}b + \frac{1}{a^2\eta}\left(\frac{204}{e^3m_1^3M_1^3}+\frac{6b^2}{e^3m_1M^3}+\frac{81M_2+9M_2m_1^2b^2}{m_2^3}\right)\,\label{eq:G_Q}
\end{align}

We note in passing that these values are not controlled by the tail behavior of the base distribution $Q$. 
This can be seen, for example, by considering
$Q(dy)=(1-\eta)\II(y\ge 0) e^{- y}dy+ \eta e^{-b} \delta_{\{-b\}}(dy)$.
Tedious calculation shows that $\lim_{u\to-\infty}\scfunc_Q(u)=\Omega(b)$ as $b\to \infty$. 
And because $Q\in  \ER(1,1)\cap \EL(1,1)$, this shows that the tail behavior of $Q$ is indeed insufficient to control the behavior of $\scfunc_Q$.

We will prove this result in three parts: {\em (i)} $\var(Q)=0$
{\em (ii)}   $\var(Q)>0$ and $\cU=[0,c_1)$, {\em (iii)} $\var(Q)>0$
and $\cU=(-c_2,0]$. The result follows from combining these cases.

The main work is to prove the result for $\cU = [0,c_1)$,
which is done in \cref{thm:self-concordance-ineq-single}.
Case {\em (iii)} is handled by ``reflection around the origin'' (\cref{coro:self-concordance-ineq-double}). Case {\em (i)} is handled in \cref{lem:mu_positive} by showing that $\scfunc_Q\equiv 0$  if $\var(Q)=0$.


We start with case {\em (i)}, the degenerate case when the variance of $Q$ is zero.
\begin{lemma}\label{lem:mu_positive}
If $\var(Q)=0$ then $\cU_Q = \R$, $Q_u=Q$ for all $u\in \RR$, and $\scfunc_Q\equiv 0$.  If $\var(Q)>0$ then $\dot\mu$ is strictly positive over the entire set $\cU_Q^\circ$.
\end{lemma}
\begin{proof}
If $\var(Q)=0$, then  $Q$ is a Dirac on some $\{v\}$. Then for all $u\in \RR$, $M_Q(u) = \exp(uv)<\infty$ hence $\psi_Q=  \log M_Q$ is supported on $\RR$ and 
    \begin{equation*}Q_u(A) =
        \begin{cases}
             \frac{1}{M_Q(u)}\exp(uv) = 1 & \text{if }v\in A\\
             0 &\text{otherwise.}
        \end{cases}
    \end{equation*}
    Hence $Q_u=Q$ and $\cQ = (Q_u)_{u\in \RR}$ is trivially self-concordant with the stretch function defined to be $\scfunc_Q\equiv 0$.
    
For the second part, by \cref{prop:moments_of_NEF}, we have that $\dot\mu(u) = \int (y-\mu(u))^2Q_u(dy)\ge 0$.
    We will show $\dot\mu(u)\neq 0$ by contradiction. Assume there exists $u_0\in \cU_Q^\circ$ such that $\dot\mu(u_0)=\int (y-\mu(u_0))^2Q_{u_0}(dy)=0$, then it follows that $Q_{u_0}(dy)$ is a Dirac on $\{\mu(u_0)\}$, which implies for all $A\in \mathcal{B}(\RR)$ 
    \[Q(A)=M_Q(u_0)\cdot Q_{u_0}(A)=\begin{cases}
        M_Q(u_0)& \text{  $\{\mu(u_0)\}\not\subseteq A$}\\
        0&\text{  otherwise}
    \end{cases}
    \]
    where $\mathcal B(\RR)$ denotes the Borel sets on $\RR$. Then it follows that $Q(dy)$ is also a Dirac on $\RR$, which contradicts that $\Var(Q)>0$.
\end{proof}

Consider now the case when $\var(Q)>0$. By the result just stated $\dot\mu$ is bounded away from zero over $\cU_Q^\circ$ and hence it is safe then to 
define $\scfunc_Q$ with the ratio $\frac{|\ddot\mu(u)|}{\dot\mu(u)}$:
\[
\selfconc_Q(u)=\frac{|\ddot\mu(u)|}{\dot\mu(u)}\,, 
\qquad u\in \cU_Q^\circ\,.
\]
Thus, in order to show our result, it suffices to show $\scfunc_Q\le \scfunc$ with the function $\scfunc$ as stated in the theorem.
Thus, we will study $\scfunc_Q$. First, notice that  for all $u\in \cU_Q^\circ$,
by \cref{prop:moments_of_NEF},
\[
\scfunc_Q(u)=
\frac{|\int (y-\mu(u))^3 Q_u(dy)|}{\int (y-\mu(u))^2 Q_u(dy)}\le\frac{\int |y-\mu(u)|^3 Q_u(dy)}{\int |y-\mu(u)|^2 Q_u(dy)}.\]

Let us now state the results that are concerned with cases {\em (ii)} and {\em (iii)} mentionned above.
For case {\em (ii)}, i.e., when $\cU = [0,c_1)$ we have the following result:
\begin{proposition}\label{thm:self-concordance-ineq-single}
Let $Q\in \ER(c_1,C_1)\cap \EL(c_2,C_2)$ and 
     $\cU=[0,c_1)$. Define $\scfunc:\cU \to \R_+$ by
    \begin{align*}
        \scfunc(u)&=\frac{3}{2}\left[2c_0\left(\frac{1}{c_1-u}\right)^2+\frac{ub}{c_1-u}+\frac{|u|b}{c_1+|u|}\right]+ \frac{1}{a^2\eta}\left(\frac{204+6u^2b^2}{c_0^3}+\frac{81C_2+9C_2u^2b^2}{(u+c_2)^3}\right)\\
        &\le \frac{3}{2}\left[2c_0\left(\frac{1}{c_1-u}\right)^2+\frac{ub}{c_1-u}\right]+G_Q(C_1,C_2,c_1,c_2)\,,
    \end{align*}
    where
     $c_0 = C_1\cdot c_1\cdot e$. Then, for appropriate values of 
     $\eta,a,b>0$ that depend on the base distribution $Q$, we have $\scfunc_Q\le \scfunc$ over $\cU$.
\end{proposition}

For case {\em (iii)}, i.e., when $\cU = (-c_2,0]$, we have the following result:
\begin{corollary}\label{coro:self-concordance-ineq-double}
Let $Q\in \ER(c_1,C_1)\cap \EL(c_2,C_2)$ and 
     $\cU=(-c_2, 0]$. 
     Define $\scfunc:\cU \to \R_+$ by
    \begin{align*}
        \scfunc(u)&=
            \frac{3}{2}\left[2c_0\left(\frac{1}{c_2-|u|}\right)^2+\frac{|u|b}{c_2-|u|}+\frac{|u|b}{c_1+|u|}\right]+\frac{1}{a^2\eta}\left(\frac{204+6u^2b^2}{c_0^3}+\frac{81C_1+9C_1u^2b^2}{(|u|+c_1)^3}\right)\\
            &\le \frac{3}{2}\left[2c_0\left(\frac{1}{c_2-|u|}\right)^2+\frac{|u|b}{c_2-|u|}\right]+G_Q(C_2,C_1,c_2,c_1)\,,
    \end{align*}
    where
     $c_0 = C_2\cdot c_2\cdot e$. Then, for appropriate values of 
     $\eta,a,b>0$ that depend on the base distribution $Q$, we have $\scfunc_Q\le \scfunc$ over $\cU$.
\end{corollary}

\subsection{Proof of \cref{thm:self-concordance-ineq-single}}
The key idea is to convert $\ddot\mu(\cdot)$ and $\dot\mu(\cdot)$ to third and second central moments respectively by \cref{prop:moments_of_NEF} and bound the third central moment in terms of the second central moment (variance).

We start by an elementary observation that says that the self-concordance properties of a NEF do not change when the base distribution is shifted by a constant:
\begin{lemma}\label{lem:shift_mean}
Let $Y\sim Q$, $c\in \R$ and define $Q^{+c}$ to be the distribution of $Y+c$.
Then, $\cU_Q = \cU_{Q^{+c}}$, $M_{Q_{+c}}(u) = e^{-u c}M_Q(u)$ for all $u\in \R$,
and $\selfconc_Q=\selfconc_{Q^{+c}}$ (here, we take $\cU = \cU_Q = \cU_{Q^{+c}}$).

\end{lemma}
\begin{proof}
By definition $M_Q(u) = \EE{e^{uY}}$ and
$M_{Q^{+c}}(u) = \EE{ e^{u(Y+c)} }$. 
Hence,
\begin{align*}
M_{Q^{+c}}(u) = \EE{ e^{u(Y+c)} } = e^{uc} \EE{e^{uY}} = e^{uc} M_Q(u)\,.
\end{align*}
This shows that $\cU_Q = \cU_{Q^{+c}}$ and that the desired relationship between $M_Q$ and $M_{Q^{+c}}$ hold.
Now, from the definition that the CGF is the logarithm of the MGF,
it follows that $\psi_{Q^{+c}}(u)=uc + \psi_Q(u)$.
Hence, $\ddot \psi_{Q^{+c}}=\ddot \psi_{Q}$ and $\dot \psi_{Q^{+c}}=\dot \psi_{Q}$, which implies that $\selfconc_Q = \selfconc_{Q^{+c}}$.
\end{proof}
Thanks to this result, from  a bound on the self-concordance function of centered distributions, we can deduce a bound on the self-concordance function of the non-centered ones.

We thus first work on establishing the bound when $Q$ is centered.

Since the theorem statement holds trivially when the variance $\var(Q)$ of $Q$ is zero, we will also assume with no loss of generality in some of the results below that $Q$ has positive variance.
\begin{lemma}\label{lem:lb_probability_base_measure}
Take a distribution $Q$ with zero mean and positive variance. 
Then, there exist $\eta>0$ and $0<a\le b$ distribution dependent constants such that $Q([-b,-a])\ge \eta$.
\end{lemma}
\begin{proof}
Since $\dot\mu(0)=\var(Q)>0$,  there exists $a>0$ and $\alpha>0$ such that $Q((-\infty, -a]))> \alpha$. As $\lim_{x\rightarrow 0}Q((-\infty, -x]))=0$, we can find some $b>0$ s.t. 
\[
Q\left(\left(-\infty,-b\right]\right) \le \frac{\alpha}{2}.
\]

This implies:
\[
   Q\left(\left[-b,-a\right]\right) \ge \frac{\alpha}{2}.
\]
The lemma thus holds with $a$ and $b$ described above, and $\eta =\frac{\alpha}{2}$.
\end{proof}
\begin{lemma}\label{lem:LBdenominator}
Take a distribution $Q$ with zero mean and positive variance.  
With $\eta,a,b$ as in \cref{lem:lb_probability_base_measure}, for all $u\in \cU_Q^\circ$, it holds that
    \begin{equation*}
        \dot\mu(u) \ge a^2\eta\frac{e^{-ub}}{M_Q(u)}\,.
    \end{equation*}
\end{lemma}
\begin{proof}
    For $a,b$ described in \cref{lem:lb_probability_base_measure}, we have:
    \begin{align*}
        \int_{\mathbb R}(y-\mu(u))^2Q_u(dy) &\ge \int_{-b}^{-a} (y-\mu(u))^2 Q_u(dy),\\
        &\ge \int_{-b}^{-a} (y-\mu(0))^2\frac{\exp(uy)}{M_Q(u)}Q(dy),\\
        &\ge a^2\frac{\exp(-ub)}{M_Q(u)}\int_{-b}^{-a} Q(dy),\\
        &\ge \frac{a^2e^{-ub}}{M_Q(u)}\eta.
    \end{align*}
The first inequality holds as $(y-\mu(u))^2$ is non-negative; the second as $\mu(u)$ increases with $u$; the third as $-b\le -a<\mu(0)$, $\mu(0)=0$ and $u\ge0$; and the last one by \cref{lem:lb_probability_base_measure}.
\end{proof}

For the upper bound, we present lemmas that bound the (upper and lower) tails of $Q_u$ and the mean $\mu(u)$.
\begin{lemma}\label{lem:UBuppertailQuexp}
Take $Q\in \ER(c_1,C_1)$ a centered distribution. Then, for all $0\le u < c_1$ and $t\ge 0$, we have
    \[
    Q_u\left((t,+\infty)\right)\le \frac{C_1e}{M_Q(u)}e^{-(c_1-u)t}\left(1+\frac{u}{c_1-u}\right)\,.
    \]
\end{lemma}

\begin{proof}\label{eq:UBuppertailinter}
    The inequality is trivially satisfied when $u=0$. Indeed, in this case $Q= Q_0$, $M_Q(0)=1$, which together
    with $Q\in \ER(c_1,C_1)$ implies the inequality.
    
    Let us now assume that  $0<u<c_1$.
    Let $v>0$ be a constant to be chosen later. Then we have that
    \begin{align}
        M_Q(u)Q_u((t,\infty))&=\int_{t}^\infty e^{uy}Q(dy)\notag\\
        &=\sum_{k= 0}^\infty\int_{t+kv}^{t+(k+1)v} e^{uy}Q(dy)\notag\\
         &\le \sum_{k= 0}^\infty e^{u(t+kv+v)}\int_{t+kv}^{+\infty} Q(dy)\notag\\
        &\le \sum_{k= 0}^\infty e^{u(t+kv+v)}C_1e^{-c_1(t+kv)}\notag\\
        &=C_1e^{-(c_1-u)t} e^{uv}\sum_{k= 0}^\infty e^{-(c_1-u)vk}.\notag
    \end{align}
    We choose $v=1/u>0$. Then $M_Q(u)Q_u((t,\infty))$ can be upper bounded by
    \begin{align}  
        M_Q(u)Q_u((t,\infty))&\le C_1e\cdot e^{-(c_1-u)t}\sum_{k= 0}^\infty e^{-\frac{c_1-u}{u}k}\notag\\
        &= C_1e\cdot e^{-(c_1-u)t}\frac{1}{1-e^{-\frac{c_1-u}{u}}}\notag\\
        &\le C_1e\cdot e^{-(c_1-u)t}\left(1+\frac{u}{c_1-u}\right)\,\notag,
    \end{align}
    where in the last inequality, we used the fact that for all $x>0$, $\frac{e^x}{e^x-1}\le 1+\frac{1}{x}$.
\end{proof}
\begin{remark}
    For $Q$ a centered exponential distribution $\mathrm{Exp}(c)$ with rate parameter $c$, the moment generating function is $M_Q(u)=e^{-\frac{u}{c}}\frac{c}{c-u}$. On the other hand,  $Q_u(t,\infty)=e^{-\frac{u}{c}}e^{-(c-u)t}$. So \cref{lem:UBuppertailQuexp} is order tight in its dependency on $c-u$. 
\end{remark}

\begin{lemma}\label{lem:UBlowertailQuexp}
Take $Q\in \EL(c_2,C_2)$ a centered distribution. Then,  for all $u,t\ge 0$, we have
    \begin{equation*}
        Q_u\left((-\infty,-t)\right)\leq \frac{1}{M_Q(u)}C_2e^{-(u+c_2)t}\,.
    \end{equation*}
\end{lemma}
\begin{proof}
We again separate the $u=0$ case. When $u=0$, $Q=Q_0$, $M_Q(0)=1$ and the inequality is equivalent to $Q\in \EL(c_2,C_2)$.

Consider now  $u>0$. Then,
    \begin{align}
   M_Q(u) Q_u\left((-\infty,-t)\right)
   =& \int_{-\infty}^{-t}e^{uy}Q(dy),\notag \\
   \leq &e^{-ut}\int_{-\infty}^{-t}Q(dy),\notag \\
   \leq &C_2e^{-ut-c_2t}.\notag 
\end{align}
\end{proof}

\begin{lemma}\label{lem:UBmuuexp}
Take  $Q\in \ER(c_1,C_1)$  with zero mean and positive variance.
     Define $c_0=c_1\cdot C_1\cdot e$. Then, for all $0\le u < c_1$, it holds that 
    \[
    0=\mu(0)\leq \mu(u)\leq c_0\left(\frac{1}{c_1-u}\right)^2\,.
    \]
\end{lemma}

\begin{proof}
For $u=0$, $\mu(0)=0$ which satisfies the inequality. We now consider $0<u<c_1$. We have
    \begin{align}
    \mu(u) =& \int_{0}^{+\infty}Q_u\left((y,+\infty)\right)dy-\int_{-\infty}^{0}Q_u\left((-\infty,-y)\right)dy\notag\\
    \leq&\int_{0}^{+\infty}Q_u\left((y,+\infty)\right)dy.\notag
\end{align}
By Lemma \ref{lem:UBuppertailQuexp}, this implies:
\begin{align*}
    \mu(u)& \le \frac{C_1e}{M_Q(u)}\int_{0}^\infty e^{-(c_1-u)t}\left(\frac{c_1}{c_1-u}\right)dt\\
    &=\frac{C_1e}{M_Q(u)}\left(\frac{c_1}{c_1-u}\right)\frac{1}{c_1-u}\\
    &= \frac{c_1C_1e}{M_Q(u)}\left(\frac{1}{c_1-u}\right)^2.
\end{align*}
By Jensen's inequality, $M_Q(u)=\E_Q[e^{uY}]\geq e^{u\E_Q[Y]}=1$, finishing the proof.
\end{proof}

\begin{proof}[Proof of \cref{thm:self-concordance-ineq-single}]
As noted beforehand, since the statement holds trivially when the variance of $Q$ is zero, we assume it is positive. By our discussion beforehand, we also assume first that $Q$ is centered, so $\mu(0)=0$.

Let $Y\sim Q_u$.
    Take $B>0$ a constant to be optimized later. We have 
\begin{align*}
    \int \left|y-\mu(u)\right|^3 \,Q_u(dy)=& \int_{0}^{+\infty} 3t^2\,\PP\left(|Y-\mu(u)|\ge t\right)dt,\\
    =& \underbrace{\int_{0}^{B} 3t^2\,\PP\left(|Y-\mu(u)|\ge t\right)dt}_{(i)}+\underbrace{\int_{B}^{+\infty} 3t^2\,\PP\left(|Y-\mu(u)|\ge t\right)dt}_{(ii)}.\\
\end{align*}
The following bound holds:
\begin{align}
    (i)\leq &3B\int_{0}^{B} t\,\PP\left(|Y-\mu(u)|\ge t\right)dt,\notag\\
    \leq & \frac{3B}{2} \dot\mu(u).\label{eq:boundi}
\end{align}
We also have:
\begin{align*}
    (ii)\leq &\underbrace{\int_{B}^{+\infty} 3t^2\,\PP\left(Y-\mu(u)\ge t\right)dt}_{(ii,a)}+\underbrace{\int_{B}^{+\infty} 3t^2\,\PP\left(Y\le -\left(t-\mu(u)\right) \right)dt}_{(ii,b)}.
\end{align*}
Set $B=2c_0\left(\frac{1}{c_1-u}\right)^2+\frac{ub}{c_1-u}+\frac{ub}{u+c_2}$ and $B'=2c_0\left(\frac{1}{c_1-u}\right)^2$. Then we can upper bound $(ii,a)$ using \cref{lem:UBuppertailQuexp}.

\begin{align*}
    (ii,a)&\le \frac{c_0}{M_Q(u)}\left(\frac{1}{c_1-u}\right)\int_B^\infty 3t^2e^{-(c_1-u)(t+\mu(u))}dt\\
    &\le \frac{c_0}{M_Q(u)}\left(\frac{1}{c_1-u}\right)\int_B^\infty 3t^2e^{-(c_1-u)t}dt\\
    &\le \frac{c_0e^{-ub}}{M_Q(u)}\left(\frac{1}{c_1-u}\right)\int_B^\infty 3t^2e^{-(c_1-u)(t-\frac{ub}{c_1-u})}dt\\
    &=\frac{c_0e^{-ub}}{M_Q(u)}\left(\frac{1}{c_1-u}\right)\int_{B-\frac{ub}{c_1-u}}^\infty 3\left(t+\frac{ub}{c_1-u}\right)^2e^{-(c_1-u)t}dt\\
    &\le \frac{6c_0e^{-ub}}{M_Q(u)}\left(\frac{1}{c_1-u}\right)\int_{B'}^\infty \left[t^2+\left(\frac{ub}{c_1-u}\right)^2\right]e^{-(c_1-u)t}dt\\
    &= \frac{6c_0e^{-ub}}{M_Q(u)}\left(\underbrace{\left(\frac{1}{c_1-u}\right)\int_{B'}^\infty t^2e^{-(c_1-u)t}dt}_{(iii,a)} + \underbrace{\left(\frac{1}{c_1-u}\right)\int_{B'}^\infty \left(\frac{ub}{c_1-u}\right)^2e^{-(c_1-u)t}dt}_{(iii,b)} \right)
\end{align*}
where in the third inequality we used the fact that $(a+b)^2\le 2a^2+2b^2$ for all $a,b\in \RR$ and that $B'<B-\frac{ub}{c_1-u}$.
We now bound $(iii,a)$. We have that $B'(c_1-u)=2c_0\left(\frac{1}{c_1-u}\right)$ and
\begin{align*}
    (iii,a)&=\frac{1}{(c_1-u)^4}e^{-B'(c_1-u)}\left(\left[B'(c_1-u)+1\right]^2+1\right)\le\frac{2}{(c_1-u)^4}e^{-B'(c_1-u)}(B'(c_1-u)+1)^2\\
    &\le \frac{4}{(c_1-u)^4}e^{-B'(c_1-u)}\left(\left[B'(c_1-u)\right]^2+1\right)\\
    &=\frac{4}{(c_1-u)^4}e^{-2c_0\left(\frac{1}{c_1-u}\right)}\left(4c_0^2\left(\frac{1}{c_1-u}\right)^2+1\right)\\
    &=\frac{16c_0^2}{(c_1-u)^6}e^{-2c_0\left(\frac{1}{c_1-u}\right)} + \frac{4}{(c_1-u)^4}e^{-2c_0\left(\frac{1}{c_1-u}\right)}\\
    &\le 16\cdot c_0^2\cdot \frac{1}{c_0^6}\left(c_0\left(\frac{1}{c_1-u}\right)\right)^6\cdot e^{-2c_0\left(\frac{1}{c_1-u}\right)} + 4\cdot \frac{1}{c_0^4}\left(c_0\left(\frac{1}{c_1-u}\right)\right)^4\cdot e^{-2c_0\left(\frac{1}{c_1-u}\right)}\\
    &\le \frac{32}{c_0^4} + \frac{2}{c_0^4}\tag{$x^6e^{-2x}\le 2$ and $x^4e^{-2x}\le 0.5$ for all $x\ge 0$.}
\end{align*}

Similarly, for $(iii,b)$, we have that
\begin{align*}
    (iii,b)&\le \frac{2}{c_1-u}\int_{B'}^\infty \left(\frac{u^2b^2}{(c_1-u)^2}\right)e^{-(c_1-u)t}dt\\
    &= \frac{2u^2b^2}{(c_1-u)^3}\frac{e^{-B'(c_1-u)}}{c_1-u}\\
    &\le \frac{2u^2b^2}{(c_1-u)^4}e^{-2c_0\left(\frac{1}{c_1-u}\right)}\\
    &\le  \frac{2u^2b^2}{c_0^4}\left(\frac{c_0}{c_1-u}\right)^4e^{-2c_0\left(\frac{1}{c_1-u}\right)}\\
    &\le \frac{u^2b^2}{c_0^4}\tag{$x^4e^{-2x}\le 0.5$ for all $x\ge 0$.}
\end{align*}
Putting the result together, $(ii,a)$ can be upper bounded as 
\begin{align*}
    (ii,a)&\le \frac{6c_0e^{-ub}}{M_Q(u)}\left(\frac{32}{c_0^4}+\frac{2}{c_0^4}+\frac{u^2b^2}{c_0^4}\right)\\
    &\le \frac{e^{-ub}}{M_Q(u)}\left(\frac{204+6u^2b^2}{c_0^3}\right).
\end{align*}
By Lemma \ref{lem:UBmuuexp}, $B\geq\mu(u)+\frac{ub}{c_2+u}$. Hence 
By Lemma \ref{lem:UBlowertailQuexp} we have:\begin{align*}
    (ii,b)\leq& \frac{C_2}{M_Q(u)}\int_{B}^{+\infty} 3t^2C_2e^{-(u+c_2)(t-\mu(u))}dt,\\
    \leq& \frac{C_2e^{-ub}}{M_Q(u)}\int_{B}^{+\infty} 3t^2C_2e^{-(u+c_2)(t-\mu(u)-\frac{ub}{u+c_2})}dt,\\
    \leq & \frac{9C_2e^{-ub}}{M_Q(u)}\int_{B-\mu(u)-\frac{ub}{u+c_2}}^{+\infty} \left(t^2+\mu(u)^2+\left(\frac{ub}{u+c_2}\right)^2\right)e^{-(u+c_2)t}dt.
\end{align*}
We now focus on 
\[\underbrace{\int_{B-\mu(u)-\frac{ub}{u+c_2}}^{+\infty} t^2e^{-(u+c_2)t}dt}_{(iv,a)}+\underbrace{\int_{B-\mu(u)-\frac{ub}{u+c_2}}^{+\infty}\mu(u)^2e^{-(u+c_2)t}dt}_{(iv,b)}+\underbrace{\int_{B-\mu(u)-\frac{ub}{u+c_2}}^\infty \left(\frac{ub}{u+c_2}\right)^2e^{-(u+c_2)t}dt}_{(iv,c)}.\]
By definition of $B$, we have that $B-\mu(u)-\frac{ub}{u+c_2}\ge c_0\left(\frac{1}{c_1-u}\right)^2=:B''$. We then have that
\begin{align*}
    (iv,a) &\le \int_{B''}^{+\infty} t^2e^{-(u+c_2)t}dt\\
    &\le \frac{e^{-B''(u+c_2)}}{(u+c_2)^3}((B''(u+c_2)+1)^2+1)\\
    &\le \frac{2e^{-B''(u+c_2)}}{(u+c_2)^3}(B''(u+c_2)+1)^2\\
    &\le \frac{4e^{-B''(u+c_2)}}{(u+c_2)^3}([B''(u+c_2)]^2+1)\\
    &\le \frac{4}{(u+c_2)^3}\left(e^{-B''(u+c_2)}[B''(u+c_2)]^2 + e^{-B''(u+c_2)}\right)\\
    &\le \frac{8}{(u+c_2)^3}
\end{align*} 
For $(iv,b)$, note that $B''\ge \mu(u)$ by \cref{lem:UBmuuexp} and we have that
\begin{align*}
    (iv,b)&\le \mu(u)^2\int_{B''}^\infty e^{-(u+c_2)t}dt\\
    &\le \frac{{B''}^2}{(u+c_2)}e^{-B''(u+c_2)}\\
    &\le \frac{1}{(u+c_2)^3}(B''(u+c_2))^2e^{-B''(u+c_2)}\\
   &\le \frac{1}{(u+c_2)^3}.
\end{align*}
For $(iv,c)$, 
\begin{align*}
    (iv,c)&\le \frac{u^2b^2}{(u+c_2)^2}\int_{B''}^\infty e^{-(u+c_2)t}dt\\
    &\le \frac{u^2b^2}{(u+c_2)^2} \frac{1}{u+c_2}e^{-B''(u+c_2)}\\
    &\le \frac{u^2b^2}{(u+c_2)^3}.
\end{align*}
Putting bounds on $(iv,a)$, $(iv,b)$ and $(iv,c)$ together, we have that
\begin{equation*}
    (ii,b)\le \frac{9C_2e^{-ub}}{M_Q(u)}\frac{9+u^2b^2}{(u+c_2)^3}\le \frac{e^{-ub}}{M_Q(u)}\frac{81C_2+9C_2u^2b^2}{(u+c_2)^3}
\end{equation*}
Combining the bounds on $(ii,a)$ and $(ii,b)$ with Lemma \ref{lem:LBdenominator} we get:
\begin{align*}
    \frac{(ii)}{\dot\mu(u)}&\leq \frac{1}{a^2\eta}\left(\frac{204+6u^2b^2}{c_0^3}+\frac{81C_2+9C_2u^2b^2}{(u+c_2)^3}\right).\\
\end{align*}
Chaining the result with the bound on $(i)$ together as well as the choice that $B=2c_0\left(\frac{1}{c_1-u}\right)^2+\frac{ub}{c_1-u}+\frac{ub}{u+c_2}$, we obtain
\begin{align*}
    \frac{\ddot\mu(u)}{\dot\mu(u)}
    &\le \frac{3}{2}\left[2c_0\left(\frac{1}{c_1-u}\right)^2+\frac{ub}{c_1-u}\right]+\frac{3}{2}b+\frac{1}{a^2\eta}\left(\frac{204+6u^2b^2}{c_0^3}+\frac{81C_2+9C_2u^2b^2}{(u+c_2)^3}\right)\\
    &\le \frac{3}{2}\left[2c_0\left(\frac{1}{c_1-u}\right)^2+\frac{ub}{c_1-u}+\frac{ub}{u+c_2}\right]+\underbrace{\frac{1}{a^2\eta}\left(\frac{204+6c_1^2b^2}{(c_1\cdot C_1\cdot e)^3}+\frac{81C_2+9C_2c_1^2b^2}{c_2^3}\right)}_{=G_Q(C_1,C_2,c_1,c_2)}.\\
\end{align*}

Let us now study the case $\mu_Q(0)\neq 0$. Let $a,b,\eta>0$ be such that $Q([-b+\mu(0),-a+\mu(0)])>\eta$ and $-a<0$. With $Q^{-\mu(0)}$ the centered version of $Q$, this gives $Q^{-\mu(0)}([-b,-a])>\eta$. We have just shown that for all $u\in[0;c_1)$,:
\begin{align*}
 \Gamma_{Q^{-\mu(0)}}(u)
    &\le \frac{3}{2}\left[2c_0\left(\frac{1}{c_1-u}\right)^2+\frac{ub}{c_1-u}+\frac{ub}{u+c_2}\right]+G_Q(C_1,C_2,c_1,c_2).
\end{align*}
By \cref{lem:shift_mean}, we have $ \Gamma_{Q^{-\mu(0)}}= \Gamma_{Q}$, hence:
\begin{align*}
 \Gamma_{Q}(u)
    &\le \frac{3}{2}\left[2c_0\left(\frac{1}{c_1-u}\right)^2+\frac{ub}{c_1-u}+\frac{ub}{u+c_2}\right]+G_Q(C_1,C_2,c_1,c_2).
\end{align*}
\end{proof}

\subsection{Proof of \cref{coro:self-concordance-ineq-double}}



\begin{lemma}\label{lem:symmetric}
    Let $Y\sim Q$, and $Q^-$ let be the distribution of $-Y$. 
    Then $\cU_Q = -\cU_{Q^{-}}$ and for any $u\in \cU_Q^\circ$, we have
    \begin{equation*}
        \selfconc_Q(u)=\selfconc_{Q^-}(-u).
    \end{equation*}
\end{lemma}
\begin{proof}
Recall that if $Q$ has zero variance, $\selfconc_Q\equiv 0$ and hence the statement is trivial.
Otherwise, for $u\in \cU_Q^\circ$, $\selfconc_Q(u) = |\dddot \psi_Q(u)|/\ddot \psi_Q(u)$.
Now, for $v\in \R$,
    \[
    M_{Q^-}(v) = \EE[ e^{(-Y)v}] = \EE[ e^{(-v)Y} ]= M_Q(-v)\,.
    \]
Hence, $\cU_Q = -\cU_{Q^-}$ and for any $v\in \cU_Q$, $\psi_Q(v)=\psi_{Q^-}(-v)$. 
Taking derivatives of both sides,
\begin{align*}
\dot\psi_Q(v) &= - \dot\psi_{Q^-}(-v)\,, \\
\ddot\psi_Q(v) &=  \ddot\psi_{Q^-}(-v)\,,\\
\dddot\psi_Q(v) &=  -\dddot\psi_{Q^-}(-v)\,,
\end{align*}
which immediately implies the statement, noting that the variance of $Q$ is positive if and only if the variance of $Q^-$ is positive.
%
\end{proof}
\begin{proof}[Proof of \cref{coro:self-concordance-ineq-double}]
Assume that $Q\in \ER(c_1,C_1)\cap \EL(c_2,C_2)$. Since the statement holds trivially when $\Var(Q)=0$, assume $\Var(Q)>0$.
Then, $Q^{-}\in \ER(c_2,C_2)\cap \ER(c_1,C_1)$.
We then get the stated result by 
 applying \cref{thm:self-concordance-ineq-single}, combined with \cref{lem:symmetric}.
 To be more specific, for all $u\in (-c_2,0]$, from these two results it follows that
\begin{align*}
    \selfconc_Q(u) &= \selfconc_{Q^-}(-u)\\
    &\le \frac{3}{2}\left[2c_0\left(\frac{1}{c_2-|u|}\right)^2+\frac{|u|b}{c_2-|u|}+\frac{|u|}{c_1+|u|}b\right]+\frac{1}{a^2\eta}\left(\frac{204+6u^2b^2}{c_0^3}+\frac{81C_1+9C_1u^2b^2}{(u+c_1)^3}\right)\\
    &\le \frac{3}{2}\left[2c_0\left(\frac{1}{c_2-|u|}\right)^2+\frac{|u|b}{c_2-|u|}+b\right]+\frac{1}{a^2\eta}\left(\frac{204+6c_2^2b^2}{c_0^3}+\frac{81C_1+9C_1c_2^2b^2}{c_1^3}\right)\,,
\end{align*}
where $c_0=C_2\cdot c_2\cdot e$.

\end{proof}
\subsection{Proof of \cref{coro:regular_NEF_subexp}}\label{section:appendix_proof_regular_NEF_subexp}
Let $\cQ$ be a regular NEF with base distribution $Q$. By definition, this means that $\cU_Q$ is an open interval. Take any $u\in \cU_Q^\circ$. There exist some $\epsilon>0$ s.t. $M_Q(u-\epsilon), M_Q(u+\epsilon)<\infty$. We also have:

\begin{align*}
     M_Q(u+\epsilon) &= \int \exp(\epsilon u)\exp(uy)Q(dy)\\
     &= \int \exp(\epsilon u)M_Q(u)Q_u(dy)\\
     &= M_Q(u)M_{Q_u}(\epsilon).
\end{align*}

Similarly, $M_{Q_u}(-\epsilon)=\frac{M_Q(u-\epsilon)}{M_Q(u)}$. By \cref{prop:subexpeq}, this implies $Q_u\in \EL\left(\frac{M_Q(u-\epsilon)}{M_Q(u)},\epsilon\right)\cap\ER\left(\frac{M_Q(u+\epsilon)}{M_Q(u)},\epsilon\right)$.

%% file: appendix/subgaussian.tex
\section{Self concordance for subgaussian base distributions}

For the convenience of the reader we restate the theorems to be proven. 
\ScgUBThm*
\ScgLBThm*

We start by introducing some notations reminiscent of the ones used for subexponential distributions. Recall that  a centered distribution $Q$ is subgaussian if and only if for some $c,C>0$, it holds that
\begin{align} \label{eq:suggausstail}
    \Prob{|X|\ge t } \le C \exp(- \frac{ct^2}{2}) \qquad \text{for all } t\ge 0
\end{align}
where $X\sim Q$.
 Let 
\[
\cG(c,C) = \{ Q \,: \,  Y=X-\EE{X} \text{ satisfies  \cref{eq:suggausstail},  where } X\sim Q \}\,.
\]

This definition is very close to that of $\ER$. In words, $\cG$ is the class of distributions over the reals whose  left and right-tail display a subgaussian decay governed with the rate parameter $c>0$ and scaling constant $C>0$.

For the interested reader, we also report here without proof some classical results on the quantitative relation between the $\MGF$ and the tail bounds of subgaussian distributions. Details can be found in the textbook \cite{rigollet2023highdimensional}.

\begin{proposition}[\cite{rigollet2023highdimensional}]
If for all $u\in \R$,$M_Q(u)\le e^{\sigma^2 u^2/2}$, then $\Prob{|X|\ge t } \le 2 \exp(- t^2/(2\sigma^2)$, and if $\Prob{|X|\ge t } \le 2 \exp(- t^2/(2\sigma^2))$, then for all $u\in \R$,
$M_Q(u)\le e^{4\sigma^2 u^2}$.
\end{proposition}

\subsection{Proof of \cref{thm:informal-self-conc-sg-ub}}

\cref{lem:mu_positive} still holds for $Q$ a subgaussian distribution, hence the theorem holds trivially for $\dot\mu(0)=0$. \cref{lem:shift_mean} can also be applied for $Q$ a subgaussian distribution. Thus, as in the proof for subexponential distributions, the first step is to show the result when $\mu(0)=0$ and $\dot\mu(0)>0$.

Note that the lower bound on $\dot\mu(u)$ of \cref{lem:LBdenominator} still holds when $\dot\mu(0)>0$. We turn to upper bounding $\int |y-\mu(u)|^3Q_u(dy)$. The following Lemmas are first steps in that direction.
\begin{lemma}\label{lem:UBuppertailQu}
    Take $Q\in \cG(c,C)$ a centered distribution. For all $u \geq 0$ and $t\ge \left(\frac{4}{c}+1\right)u+\frac{4}{c}$, we have:
    \begin{equation*}
    Q_u\left((t,+\infty)\right)\leq \frac{\mfC_1}{M_Q(u)}e^{-\frac{c t^2}{4}},
    \end{equation*}
    where $\mfC_1 = C\left(1+\sqrt{\frac{\pi}{c}}\right)$.
\end{lemma}
\begin{proof}
    We have 
        \begin{align}
       M_Q(u) Q_u\left((t,+\infty)\right)=& \int_{t}^{+\infty}e^{uy}Q(dy),\notag \\
       =&\sum_{k=0}^{+\infty}\int_{t+k}^{t+k+1}e^{uy}Q(dy), \notag \\
       \leq& \sum_{k=0}^{+\infty}e^{u(t+k+1)}\int_{t+k}^{t+k+1}Q(dy), \notag \\
       \leq& \sum_{k=0}^{+\infty}e^{u(t+k+1)}Q\left([t+k;+\infty)\right),\notag \\
       \leq& C\sum_{k=0}^{+\infty}e^{-\left(\frac{c(t+k)^2}{2}-u(t+k+1)\right)}. \label{eq:UBuppertailinter}
    \end{align}
    The first inequality holds as $u\ge 0$, the second by subgaussianity assumption. As $t\geq \left(\frac{4}{c}+1\right)u+\frac{4}{c}$, we have $u\leq \frac{c}{4}t$ and:
    \begin{equation}\label{eq:boundexponentinter}
    t^2\geq\left(\frac{4}{c}+1\right)u t=  \frac{4t}{c}u+\underbrace{t}_{\geq \frac{4}{c}}u\ge \frac{4}{c}(t+1)u. 
    \end{equation}
    This implies $u(t+1)\leq \frac{c}{4}t^2$ on top of $u\leq \frac{c}{4}t$. Then
    \[
    u(t+1+k)\leq \frac{c}{4}\left(t^2+kt\right)\leq \frac{c}{4}\left(t+k\right)^2.
    \]
    Reinjecting in \cref{eq:UBuppertailinter}, we get:
    \begin{align*}
       M_Q(u) Q_u\left((t,+\infty)\right)
       \leq& C\sum_{k=0}^{+\infty}e^{-\frac{c(t+k)^2}{4}},\\
       \leq& Ce^{-\frac{ct^2}{4}}\sum_{k=0}^{+\infty}e^{-\frac{ck^2}{4}}. \\
       \leq& Ce^{-\frac{ct^2}{4}}\left(1+\int_0^\infty e^{-\frac{cx^2}{4}}dx\right)\\
       \le& C\left(1+\sqrt{\frac{\pi}{c}}\right)e^{-\frac{ct^2}{4}}.\label{eq:UBuppertailinter}
    \end{align*}
\end{proof}
\begin{lemma}\label{lem:UBlowertailQu}
   Take $Q\in \cG(c,C)$ a centered distribution. For all $ u,t\ge 0$, we have:
    \begin{equation*}
           Q_u\left((-\infty,-t)\right)\leq \frac{C}{M_Q(u)}e^{-ut-\frac{c t^2}{2}}.
    \end{equation*}
\end{lemma}
\begin{proof}
    We have that 
    \begin{align}
       M_Q(u) Q_u\left((-\infty,-t)\right)
       =& \int_{-\infty}^{-t}e^{uy}Q(dy),\notag \\
       \leq & e^{-ut}\int_{-\infty}^{-t}Q(dy),\notag \\
         \leq & Ce^{-ut-\frac{c t^2}{2}}\,,
\end{align}
where the last inequality holds by subgaussianity of $Q$.
\end{proof}
\begin{lemma}\label{lem:UBmuu}
    Take $Q\in \cG(c,C)$ a centered distribution. For all $u \geq 0$ the following holds:
    \begin{equation*}
    0=\mu(0)\leq \mu_Q(u)\leq \left(\frac{4}{c}+1\right)u+\frac{4}{c}+\mfC_3e^{-\frac{4u^2}{c}},
    \end{equation*}
    where $\mfC_3 = \frac{\sqrt{\pi}}{\sqrt c}\mfC_1$.
\end{lemma}
\begin{proof}
    We have:
\begin{align}
    \mu(u) =& \int_{0}^{+\infty}Q_u\left((y,+\infty)\right)dy-\int_{-\infty}^{0}Q_u\left((-\infty,-y)\right)dy\notag\\
    \leq&\int_{0}^{+\infty}Q_u\left((y,+\infty)\right)dy\notag\\
    \leq &\left(\frac{4}{c}+1\right)u+\frac{4}{c}+\int_{\left(\frac{4}{c}+1\right)u+\frac{4}{c}}^{+\infty}Q_u\left((y,+\infty)\right)dy\notag.
\end{align}
By \cref{lem:UBuppertailQu}, this implies:
\begin{align}
    \mu(u) 
    \leq &\left(\frac{4}{c}+1\right)u+\frac{4}{c}+\frac{\mfC_1}{M_Q(u)}\int_{\left(\frac{4}{c}+1\right)u+\frac{4}{c}}^{+\infty}e^{-\frac{cy^2}{4}}dy\notag\\
    \leq & \left(\frac{4}{c}+1\right)u+\frac{4}{c}+\frac{\mfC_1 e^{-c\frac{\left(\left(\frac{4}{c}+1\right)u+\frac{4}{c}\right)^2}{4}}}{M_Q(u)}\int_{0}^{+\infty}e^{-\frac{cy^2}{4}}dy\notag\\
    \leq& \left(\frac{4}{c}+1\right)u+\frac{4}{c}+\frac{\mfC_1}{M_Q(u)}e^{-\frac{4u^2}{c}}\int_{0}^{+\infty}e^{-\frac{cy^2}{4}}dy\notag.
\end{align}
By Jensen's inequality, $M_Q(u)=\E_Q[e^{uY}]\geq e^{u\E_Q[Y]}=1$. Noting that $\int_{0}^{+\infty}e^{-\frac{cy^2}{4}}dy\le \frac{\sqrt {\pi}}{\sqrt{c}}$ terminates the proof.
\end{proof}
\begin{proof}[Proof of \cref{thm:informal-self-conc-sg-ub}] 
As previously noted, we first show the result when $\mu(0)=0$ and $\dot\mu(0)>0$. We start with $u> 0$, then by \cref{lem:symmetric} will extend the bound to $u< 0$.

Take $u>0$ and $B>0$ a constant to be optimized later. We have:
\begin{align*}
    \int \left|Y-\mu_Q(u)\right|^3 \,Q_u(dy)=& \int_{0}^{+\infty} 3t^2\,\PP\left(|Y-\mu_Q(u)|\ge t\right)dt,\\
    \leq& \underbrace{\int_{0}^{B} 3t^2\,\PP\left(|Y-\mu_Q(u)|\ge t\right)dt}_{(i)}+\underbrace{\int_{B}^{+\infty} 3t^2\,\PP\left(|Y-\mu_Q(u)|\ge t\right)dt}_{(ii)}.\\
\end{align*}

The first equality is a classical result on the relationship between moments and tails of r.v., see for instance Exercise 1.2.3 of \cite{vershynin2018high}. 
The following bound holds:
\begin{align}
    (i)\leq &3B\int_{0}^{B} t\,\PP\left(|Y-\mu_Q(u)|\ge t\right)dt,\notag\\
    \leq & \frac{3B}{2} \dot\mu_Q(u).\label{eq:boundi}
\end{align}
We also have:
\begin{align*}
    (ii)\leq &\underbrace{\int_{B}^{+\infty} 3t^2\,\PP\left(Y-\mu_Q(u)\ge t\right)dt}_{(ii,a)}+\underbrace{\int_{B}^{+\infty} 3t^2\,\PP\left(Y\le -\left(t-\mu_Q(u)\right) \right)dt}_{(ii,b)}.
\end{align*}
Set $B= \mu_Q(u)+\left(\frac{4}{c}+1\right)u+\frac{4}{c}+b$, where $b$ was defined in \cref{lem:LBdenominator}. As $B\ge\left(\frac{4}{c}+1\right)u+\frac{4}{c}$, by Lemma \ref{lem:UBuppertailQu}, we have:

\begin{align*}
    (ii,a)\leq&  \int_{B}^{+\infty} 3t^2\,Q_u\left((t,+\infty)\right)dt,\\
    \leq &\frac{\mfC_1}{M_Q(u)}\int_{B}^{+\infty} 3t^2e^{-\frac{ct^2}{4}}dt,\\
    \leq& \frac{3\mfC_1}{M_Q(u)}\int_{0}^{+\infty} (B+t)^2e^{-\frac{c(B+t)^2}{4}}dt,\\
    \leq& \frac{6\mfC_1e^{-\frac{cB^2}{4}}}{M_Q(u)}\int_{0}^{+\infty} (B^2+t^2)e^{-\frac{ct^2}{4}}dt,\\
    =& \frac{6\mfC_1e^{-\frac{cB^2}{4}}}{M_Q(u)}(B^2+\frac{2}{c})\int_{0}^{+\infty} e^{-\frac{ct^2}{4}}dt,\\
    =& \frac{6\mfC_1}{M_Q(u)}\sqrt{\frac{\pi}{c}}(B^2+\frac{2}{c})e^{-\frac{cB^2}{4}}.\\
\end{align*}
where the first equality holds as for any $a>0$, using integration by part we have that
\begin{equation}
    \int_{0}^{+\infty}e^{-at^2}dt=2a\int_{0}^{+\infty}t^2e^{-at^2dt}.\label{eq:sg_IBP}
\end{equation}
Denote $B'= B-\mu_Q(u)=\left(\frac{4}{c}+1\right)u+\frac{4}{c}+b$.  By \cref{lem:UBlowertailQu} we have:
\begin{align*}
    (ii,b)\leq& \int_{B}^{+\infty} 3t^2\frac{\mfC_2}{M_Q(u)}e^{-u(t-\mu_Q(u))-\frac{c (t-\mu_Q(u))^2}{2}}dt,\\
    =& \frac{3\mfC_2}{M_Q(u)}\int_{B'}^{+\infty} \left(t+\mu_Q(u)\right)^2e^{-ut-\frac{c t^2}{2}}dt,\\
    \le& \frac{3\mfC_2e^{-\frac{4}{c}u^2-bu}}{M_Q(u)}\int_{B'}^{+\infty} \left(t+\mu_Q(u)\right)^2e^{-\frac{c t^2}{2}}dt,\\
    \leq &\frac{3\mfC_2e^{-\frac{4}{c}u^2-bu}}{M_Q(u)}e^{-\frac{c (B')^2}{2}}\int_{0}^{+\infty} \left(t+B'+\mu_Q(u)\right)^2e^{-\frac{c t^2}{2}}dt,\\
    \leq &\frac{6\mfC_2e^{-\frac{4}{c}u^2-bu}}{M_Q(u)}e^{-\frac{c (B')^2}{2}}\int_{0}^{+\infty} \left(t^2+B^2\right)e^{-\frac{c t^2}{2}}dt,\\
     \leq &\frac{6\mfC_2e^{-\frac{4}{c}u^2-bu}}{M_Q(u)}e^{-\frac{c (B')^2}{2}} \left(\frac{1}{c}+B^2\right)\int_{0}^{+\infty}e^{-\frac{c t^2}{2}}dt,\\
     \leq &\frac{6e^{-\frac{12}{c}u^2-bu}}{M_Q(u)} \left(\frac{1}{c}+B^2\right)\underbrace{\int_{0}^{+\infty}\mfC_2e^{-\frac{c t^2}{2}}dt}_{\mfC_4}.
\end{align*}
where in the second line we use change of variable; third line we use the fact that $B'>\frac{4}{c}u+b$, hence $e^{-ut}\le e^{-B'u}\le e^{-\frac{4}{c}u^2-ub}$ for all $t\ge B'$; fourth line change of variable and the fact that $-(a+b)^2\le -a^2-b^2$ for all $a,b\ge 0$; fifth line the fact that $(a+b)^2\le 2a^2+2b^2$ for all $a,b\in \RR$; sixth line \cref{eq:sg_IBP}; and seventh line $B'>\frac{4}{c}u$.

Combining the bounds on $(ii,a)$ and $(ii,b)$ with Lemma \ref{lem:LBdenominator_subg} we get:
\begin{align*}
    \frac{(ii)}{\dot\mu_Q(u)}\leq \frac{6}{a^2\eta}\left(\mfC_4\left(\frac{1}{c}+B^2\right)e^{-\frac{12}{c}u^2}+\mfC_3\left(\frac{2}{c}+B^2\right)e^{-\frac{cB^2}{4}+ub}\right).
\end{align*}
As $B\geq \frac{4}{c}u+b$, we have $\frac{cB^2}{4}+ub\geq \frac{4}{c}u^2$. This implies
\begin{align*}
    \frac{(ii)}{\dot\mu_Q(u)}\leq \frac{6(\mfC_3+\mfC_4)}{a^2\eta}\left(\frac{2}{c}+B^2\right)e^{-\frac{4}{c}u^2}.
\end{align*}

With Equation \ref{eq:boundi}, we get
\begin{align*}
    \selfconc_Q(u)\leq \frac{6(\mfC_3+\mfC_4)}{a^2\eta}\left(\frac{2}{c}+B^2\right)e^{-\frac{4}{c}u^2}+\frac{3B}{2}.
\end{align*}

By Lemma \ref{lem:UBmuu}, $\mu_Q(u)\leq \left(\frac{4}{c}+1\right)u+\frac{4}{c}+\mfC_3$, which implies :
\[
B\leq 2\left(\frac{4}{c}+1\right)u+\frac{8}{c}+b+\mfC_3.
\]

For any constants $w_1,w_2>0$, we have $w_1 u^2 e^{-w_2 u^2}\leq \frac{w_1}{\sqrt{w_2}}u$. This implies that for  $u>0$:
\begin{align*}
    \selfconc_Q(u) = O(u).
\end{align*}

Note that if $Q \in \cG(c,C)$, with $Q^-$ the distribution of $-Y$, $Y\sim Q$, we have $Q^- \in \cG(c,C)$. By \cref{lem:symmetric}, that implies:
\begin{align*}
    \selfconc_Q(-u) = \selfconc_{Q^-}(u)=O(u).
\end{align*}
Therefore, for $u \in \mathbb{R}$:
\begin{align*}
    \selfconc_Q(u) = O(|u|).
\end{align*}

 As, by \cref{lem:mu_positive}, $\selfconc_Q$ 
 is constant for $\dot\mu(0)=0$, it remains now only to show the theorem if $\mu(0)\neq 0$ and $\dot\mu(0)> 0$. We have just shown
\begin{align*}
    \selfconc_{Q^{-\mu(0)}}(u) = O(|u|),
\end{align*}
with $Q^{-\mu(0)}$ the centered version of $Q$. by \cref{lem:shift_mean}, we have  $\selfconc_Q=\selfconc_{Q^{-\mu(0)}}$. Hence,
for any $Q \in \cG(c,C)$:
 \begin{align*}
    \selfconc_Q(u) = O(|u|).
\end{align*}

\end{proof}
\subsection{Proof for \cref{thm:informal-self-conc-lb}}
We construct a distribution $Q$ s.t. for any $s>0$,  we can find some $u >s$ with :
\[
\frac{\int (y-\mu(u))^3Q_u(dy)}{\int (y-\mu(u))^2Q_u(dy)}\geq 0.038 u.
\]
Let $Q$ the base measure be
\begin{equation}\label{eqn:counter-example}
    Q(dy) = \sum_{i\geq 1} p_i \delta_{2^i}(dy),
\end{equation}
with
\begin{equation}
 p_i =
    \begin{cases}
      C_1\exp(-4^i) & \text{if $i$ is even},\\
      \frac{C_1}{4}\exp(-3 \times 4^{i-1})  & \text{if  $i$ is odd},
    \end{cases}  
\end{equation}
where $C_1$ is a normalizing constant. By definition, we have that $Q_u(dy) = \sum_{i\ge 1}q_i\delta_{2^i}(dy)$ where
\begin{equation}
    q_k=
    \begin{cases}
        \frac{C_1}{M_Q(u)}\exp(u2^{k})\exp(-4^k) &\text{if $k$ is even}\\
        \frac{C_1}{4M_Q(u)}\exp(u2^{k})\exp(-3\times 4^{k-1})&\text{if $k$ is odd.}
    \end{cases}
\end{equation}
We are going to inspect $\frac{|\EE[(Y-\mu(u))^3]|}{\EE[(Y-\mu(u))^2]}$ for $i$ a positive even number large enough and $u=2^{i+1}$. By definition, we have that $Q_u(dy) = \sum_{i\ge 1}q_i\delta_{2^i}(dy)$ where
\begin{equation}\label{eq:lb_tilted}
    q_k=
    \begin{cases}
        \frac{C_1}{M_Q(u)}\exp(2^{i+k+1})\exp(-4^k) &\text{if $k$ is even},\\
        \frac{C_1}{4M_Q(u)}\exp(2^{i+k+1})\exp(-3\times 4^{k-1})&\text{if $k$ is odd}.
    \end{cases}
\end{equation}
First, we remark a few simple equalities and inequalities that will prove useful in the subsequent computations.

Not that $\mathbb{P}\left(U=2^i\right)= \frac{C_1}{M_Q(u)}e^{4^i}$, hence:
\begin{equation}\label{eq:ivsi+1}
    \frac{\mathbb{P}\left(U=2^{i+1}\right)}{\mathbb{P}\left(U=2^i\right)}= \frac{1}{4}e^{2^{2(i+1)}-3\times 4^{i}-4^i}=\frac{1}{4}.
\end{equation}
This implies: 
    \begin{equation}\label{eq:lb_prob_U=u}
 \mathbb{P}\left(U=2^i\right)\leq \frac{4}{5}\,,
    \end{equation}
and, combining with Equation \cref{eq:lb_tilted},  
\begin{equation}\label{lb:ratio}
    \frac{\mathbb{P}\left(U=2^j\right)}{\mathbb{P}\left(U=2^i\right)}=
    \begin{cases}
        e^{2^{i+j+1}-4^j-4^i} &\text{if }j \text{ is even},\\
        \frac{1}{4}e^{2^{i+j+1}-3 \times 4^{j-1}-4^i}&\text{if }j \text{ is odd}.
    \end{cases}
\end{equation}
From this last equation we get upper bound:
\begin{equation}\label{eq:ubratio}
    \frac{\mathbb{P}\left(U=2^j\right)}{\mathbb{P}\left(U=2^i\right)}\leq e^{2^{i+j+1}-3\times4^{j-1}-4^i}.
\end{equation}

The two next Lemma combined show that for any even $i\geq 4$, $1.24 \times 2^i\leq \mu(u)\leq 1.26 \times2^{i}$.
\begin{lemma}\label{lem:lower_mean_bounds}
    Let $i\geq 4$ be an even number and $u=2^{i+1}$. For random variable $U \sim Q_u(dy)$, it holds that:
    \begin{equation}
        \mu(u)\ge1.24 \times  2^i.
    \end{equation}
\end{lemma}
\begin{proof}
Consider $j<i$, $j$ even. By \cref{eq:ubratio}, 
\[
\frac{\mathbb{P}\left(U=2^j\right)}{\mathbb{P}\left(U=2^i\right)}\leq e^{2^{i+j+1}-3\times4^{j-1}-4^i}.
\]
Denote $f_1(j)=2^{i+j+1}-3\times4^{j-1}-4^i$. The monotonicity of the right hand side, $e^{f_1(j)}$, is the same as $f_1(j)$. Take derivative of $f_1(j)$, we have that
\begin{equation*}
    f_1'(j)=2^j\cdot 2^{i+1}\ln 2 - 3\times 4^{j-1} \ln 4 = 2^{j}(2^{i+1}\ln 2 - 3\times 2^{j-2}\ln 4)=2^{j}\ln 4(2^{i} - 3\times2^{j-2}).
\end{equation*}
Since $i>j$ we have that $f_1'(j)\ge 0$ for $j\in [0,i]$.
Hence the right hand side increases with $j$, as $j\leq i-1$, we have that
\begin{equation}
    \frac{\mathbb{P}\left(U=2^j\right)}{\mathbb{P}\left(U=2^i\right)}\leq e^{-3\times 4^{i-2}}.
\end{equation}

This implies:
\begin{equation}\label{eq:P_U<2^i}
     \frac{\mathbb{P}\left(U<2^i\right)}{\mathbb{P}\left(U=2^i\right)}=  \frac{\sum_{j<i}\mathbb{P}\left(U=2^j\right)}{\mathbb{P}\left(U=2^i\right)} \leq (i-1) e^{-3\times 4^{i-2}}\le 0.001\,.
\end{equation}
We now prove the lower bound on the mean.
\begin{align*}
    \mu(u) \geq& 2^{i+1} \mathbb{P}(U=2^{i+1})+2^{i} \mathbb{P}(U=2^{i})+2^{i} \mathbb{P}(U>2^{i+1})\\
    = & 2^{i+1}\mathbb{P}(U=2^{i+1})+2^i\mathbb{P}(U=2^i)+2^i(1-\mathbb{P}(U<2^{i})-\mathbb{P}(U=2^{i})-\mathbb P(U=2^{i+1}))\\
    \geq& 2^{i} +\left[2^{i+1}-2^{i}\right]\mathbb{P}(U=2^{i+1})-2^i\mathbb{P}(U<2^{i})\\
    =& 2^{i} +2^i\mathbb{P}(U=2^{i})\left(\frac{\mathbb{P}(U=2^{i+1})}{\mathbb{P}(U=2^{i})}-\frac{\mathbb{P}(U<2^{i})}{\mathbb{P}(U=2^{i})}\right)\\
    \geq & 2^i+ 2^i\left(\frac{1}{4}-0.001\right)\mathbb{P}(U=2^{i})\\
    \geq& 1.24 \times 2^i,
\end{align*}
where the fourth line holds because of \cref{eq:P_U<2^i,eq:ivsi+1}.
\end{proof}
\begin{lemma}\label{lem:ubmean}
    Let $i$ be a positive even number and $u=2^{i+1}$. For random variable $U \sim Q_u(dy)$, it holds that:
    \begin{equation}
        \mu(u)\le 1.26\cdot 2^i\,.
    \end{equation}
\end{lemma}
\begin{proof}
    By \cref{eq:ubratio}, for any  $j\geq i+2$, we have: 
\begin{align}
     \frac{\mathbb{P}\left(U=2^{j}\right)}{\mathbb{P}\left(U=2^i\right)}&\leq e^{2^{i+j+1}-3\times 4^{j-1}-4^i}=e^{-2^{j-1}\left(3\times 2^{j-1}-2 \times 2^{i+1}\right)-4^i},\nonumber\\
     &\leq e^{-4^i-2^{j-1}}\label{eq:bound2^j}.
\end{align}

This implies:
\[
\sum_{j> i+1}2^j \frac{\mathbb{P}_u\left(U=2^{j}\right)}{\mathbb{P}\left(U=2^i\right)} \leq e^{-4^i}\sum_{j> i+1}2^je^{-2^{j-1}}\leq \left(\sum_{j=4}^\infty 2^j e^{-2^{j-1}}\right) e^{-4^i}.
\]
We turn to bounding $\sum_{j=4}^\infty 2^j e^{-2^{j-1}}$:
\begin{align*}
    \sum_{j=4}^\infty 2^j e^{-2^{j-1}} &= 2\sum_{j=4}^\infty 2^{j-1} e^{-2^{j-1}} = 2\sum_{j=3}^\infty 2^{j} e^{-2^{j}},\\
    &\le 2\int_{2}^\infty 2^xe^{-2^x}dx,\\
    &=2\int_{2}^\infty \frac{1}{y\ln 2}ye^{-y}dy,\\
    &=\frac{2}{e^4\ln 2},
\end{align*}
where the second inequality holds as $2^x e^{-2^x}$ is decreasing for $x>0$. Hence, for any $i\geq 2$:

\begin{equation}\label{eq:boundlargenumbers}
\sum_{j> i+1}2^j \frac{\mathbb{P}_u\left(U=2^{j}\right)}{\mathbb{P}\left(U=2^i\right)}\leq \frac{2}{e^4\ln 2}e^{-4^i}\leq \frac{1}{100}2^i.
\end{equation}

We are now ready to upper bound $\mu_u$:
\begin{align*}
   \mu_u &\leq 2^{i-1}\mathbb P(U<2^i)+2^i\mathbb P(U=2^i)+2^{i+1}\mathbb P(U=2^{i+1})+\sum_{j>i+1}2^j\mathbb P(U=2^j) \\
   &\le 2^i P(U=2^i) \left(\frac{1}{2}\frac{P(U<2^i)}{P(U=2^i)}+1+ 2\frac{P(U=2^{i+1})}{P(U=2^i)}+\frac{1}{2^i}\sum_{j> i+1}2^j \frac{\mathbb{P}_u\left(U=2^{j}\right)}{\mathbb{P}\left(U=2^i\right)}\right)\\
      &\le 2^i \frac{4}{5} \left(\frac{1}{2}\times 0.001+1+ \frac{1}{2}+\frac{1}{100}\right)\\
   &<1.26 \cdot 2^i\,,
\end{align*}

where we get the third inequality from \cref{eq:lb_prob_U=u,eq:P_U<2^i,eq:ivsi+1,eq:boundlargenumbers}.

\end{proof}
    \begin{proof}[Proof for \cref{thm:informal-self-conc-lb}]
        Let $U\sim Q_u$ where $u=2^{i+1}$ for $i\ge 4$ an even number. We first derive an upper bound on the variance $\EE[(U-\mu(u))^2]\le \EE[U^2]$.
        \begin{align}
        \mathbb E[U^2]&\leq 2^{2i}\cdot \mathbb P(U<2^i)+2^{2i}\cdot \mathbb P(U=2^i)+2^{2i+2}\cdot \mathbb P(U=2^{i+1})+\sum_{j>i+1}2^{2j}\mathbb P(U=2^j)\nonumber\\
   &\leq 2^{2i}P(U=2^i)\left(\frac{\mathbb P(U<2^i)}{\mathbb P(U=2^i)}+1+4\frac{\mathbb P(U=2^{i+1})}{\mathbb P(U=2^i)}+\frac{1}{2^{2i}}\sum_{j>i+1}2^{2j}\frac{\mathbb P(U=2^j)}{\mathbb P(U=2^i)}\right)\nonumber\\
     &\leq 2^{2i}\frac{4}{5}\left(0.112+1+1+\frac{e^{-4^i}}{2^{2i}}\sum_{j>i+1}2^{2j}e^{-2^{j-1}}\right)\,,\label{eq:U2inter}
    \end{align}
    where the last inequality is obtained from \cref{eq:lb_prob_U=u,eq:P_U<2^i,eq:ivsi+1,eq:bound2^j}.

    We inspect the infinite series in the above inequality  . Note that $2^{2j}e^{-2^{j-1}}$ is decreasing for $j\ge 1$.
    \begin{align*}
        \sum_{j>i+1}2^{2j}e^{-2^{j-1}}&\leq4\sum_{j\ge 2}4^je^{-2^j}\\
        &\le 4\int_{1}^\infty 4^xe^{-2^x}dx\\
        &=4\int_{2}^\infty \frac{1}{y\ln 2}y^2e^{-y}dy\\
        &=\frac{12}{\ln 2 \cdot e^2}.
     \end{align*}
For all $i\ge 1$, it follows that 
    \begin{equation*}
        \frac{e^{-4^i}}{2^{2i}}\frac{12}{\ln 2 \cdot e^2}\le 0.02.
    \end{equation*}
    Plugging into Equation \ref{eq:U2inter},  we have:
    \begin{equation}\label{eq:ubu2}
        \mathbb{E}[(U-\mu(u))^2]\leq 1.7\times 2^{2i},
    \end{equation}
    On the other hand, by \cref{lem:lower_mean_bounds,lem:ubmean}, we have that
    \begin{equation*}
        U-\mu(u)\ge \begin{cases}
            \underbrace{0.74 \cdot 2^i}_{A_{\mathrm{large}}}&\text{if }U\ge 2^{i+1}\\
            \underbrace{-0.26 \cdot 2^i}_{A_{\mathrm{medium}}}&\text{if }U= 2^{i}\\
            \underbrace{-1.26 \cdot 2^i}_{A_\mathrm{small}}&\text{if }U< 2^{i}.
        \end{cases}
    \end{equation*}
    Then we can lower bound the third central moment $\EE[|U-\mu(u)|^3]$.
     \begin{align*}
     \mathbb E[(U-\mu(u))^3]&\ge \mathbb P(U\ge 2^{i+1})A_{\mathrm{large}}^3 + \mathbb P(U= 2^i)A_{\mathrm{medium}}^3 + \mathbb P(U<2^i)A_{\mathrm{small}}^3\\
     &=[1-\mathbb P(U=2^i)-\mathbb P(U<2^i)]\cdot A_{\mathrm{large}}^3 + \mathbb P(U= 2^i)A_{\mathrm{medium}}^3 + \mathbb P(U<2^i)A_{\mathrm{small}}^3\\
     &=(0.74\cdot 2^i)^3 - \mathbb P(U=2^i)[(0.74\cdot2^i)^3+(0.26\cdot 2^i)^3]-\mathbb P(U<2^i)((0.74\cdot2^i)^3+(1.26\cdot 2^i)^3)\\
     &\geq(2^i)^3\left[0.74^3-0.8(0.74^3 + 0.26^3) - 0.8\times0.001(1.26^3+0.74^3)\right]\\
     &\ge 0.065(2^{i})^3.
     \end{align*}
    where the third inequality holds by \cref{eq:lb_prob_U=u,eq:P_U<2^i}.
Combining this last bound with \cref{eq:ubu2}, we obtain:
    \[
    \frac{ \mathbb{E}[(U-\mu(u))^3]}{ \mathbb{E}[(U-\mu(u))^2]}\geq \frac{0.065 }{1.7}2^i=0.038 u.
    \]
\end{proof}

%% file: appendix/bandits.tex
\section{Bandit algorithm}
In this section, we present our analysis of \cref{alg:OFU-GLB}. The section is organized as follows. In  \cref{section:conf_set_construction}, we detail how we construct the confidence set. 
The proof that  $\theta_\star$ lies in the confidence set with high probability will be presented only in  \cref{section:conf_set_appendix} so that we can continue in \cref{section:regret_upper_bound} with the proof of the main result, \cref{thm:regret_bound}, bounding the regret. This result differs from 
 \cref{thm:informal-regret-bound} by providing additional detail about the choice of the parameters in the algorithm. 
The proof presented in \cref{section:regret_upper_bound} requires a number of technical lemmas that are presented as the proof develops. 
The proofs of these are postponed to subsequent sections.
Before the proof of these, we devote the next section  (\cref{section:self-conc-control}) to technical results on consequences of self-concordance which will be useful for the rest of the proofs.
This is followed in \cref{section:conf_set_appendix}
by the proof that the confidence set constructed indeed has the required coverage.
The next section (\cref{section:proof_of_conf_set_diam_H}) is devoted to proving \cref{lem:conf_set_diam_H} (``ellipsoidal diameter bound on the confidence set''), which is one of the two key results required for the proof of the main regret bound (besides the result on the coverage of the confidence set).
A self-bounding property of self-concordance functions (\cref{lem:dot-mu-sum}), which is the second main ingredient of the regret bound proof is shown in  \cref{sec:dot-mu-sum}.
Finally, for completeness, we present the (well known) elliptical potential lemma 
(stated here as \cref{lem:epl}) 
in \cref{section:aux_lem}. 

\subsection{Constructing the confidence set}\label{section:conf_set_construction}
The confidence set construction is based on first obtaining the parameter vector $\hat\theta_t$. This parameter vector is chosen to be the minimizer of the regularized negative log-likelihood function:
$\hat\theta_t = \argmin_{\theta\in \R^d} \cL(\theta;\cD_t)$ where $\cD_t = ((X_i,Y_i))_{i=1}^{t-1}$ is the data available in step $t$ and 
\begin{align}
   \cL(\theta;\cD_t,\lambda)
   =\frac{\lambda}{2}\|\theta\|^2 - \sum_{i=1}^{t-1}\log q(Y_i;X_i^\top\theta) \label{eq:nll}
\end{align}
where $\lambda>0$ is a tuning parameter (to be chosen later) and
$q(y;u) = \frac{dQ_u}{dQ}(y)$ is the density of $Q_u$ with respect to $Q$ at $y\in \R$.
It should be clear from the definitions that $q$ is well-defined.
The purpose of regularization is to ensure that the loss function has a unique optimizer even in the data poor regime. \todoc{Not quite a valid reason, but OK for submission.}

For the construction of the confidence set it will be useful to derive an equivalent expression for the loss $\cL$. For this, first note that the density $q$ satisfies 
$q(y;u)=\exp( y u -\psi_Q(u))$. 
Plugging this into the definition of $\cL$, we get
\begin{align*}
   \cL(\theta;\cD_t,\lambda)&=\frac{\lambda}{2}\|\theta\|^2+\sum_{i=1}^{t-1}
   (\psi(X_i^\top\theta)-Y_iX_i^\top\theta)\,.
\end{align*}
As it is well known, $\psi$ is a convex function of its argument (Theorem 1.13 of \citep{brown1986}) 
and hence $\theta\mapsto \cL(\theta;\cD)$ is strictly convex provided that $\lambda>0$.

For the confidence set construction we will need the non-constant part of the gradient of $\cL(\cdot;\cD_t)$, which we denote by $g_t$. We will also need the curvature of $\cL(\cdot;\cD_t)$, which we denote by $H_t$. These are
\begin{align*}
	g_t(\theta) &= \sum_{i=1}^{t-1} \mu(X_i^\top\theta)X_i+\lambda\theta
	\quad \text{ so that }  \quad \nabla_\theta\cL(\theta;\cD_t) =g_t(\theta)-\sum_{i=1}^{t-1} X_iY_i\,\,,\\
	\qquad \text{ and }\\
   H_t(\theta)&
    =\nabla^2_\theta\cL(\theta;\cD_t)= \lambda I +
		 \sum_{i=1}^{t-1} \dot\mu(X_i^\top\theta)X_iX_i^\top\,.
\end{align*}
The minimizer $\hat\theta_t = \argmax_{\theta\in \RR^d}\cL(\theta;\cD,\lambda)$ has the property that
\begin{equation*}
   \frac{\partial \cL(\theta;\cD_t,\lambda)}{\partial \theta}\Bigg\vert_{\theta=\hat\theta_t} =0\,.
\end{equation*}
This implies that
\begin{align*}
     g_t(\hat\theta_t)-\sum_{i=1}^t X_iY_i = 0.
\end{align*}
With this, we can introduce our confidence set construction, which is based on the work of 
\cite{janz2023exploration}.
For $\delta\in (0,1]$, we let
\begin{equation}\label{eq:conf_set}
   \cC_t^{\delta}(\hat\theta_t)=\left\{\theta\in \Theta\,:\,\left\|g_t(\theta)-g_t(\hat\theta_t)\right\|_{H_t^{-1}(\theta)}\le \gamma_t(\delta)\right\}\,,
\end{equation}
where for $M$ to be chosen later (in \cref{lem:confset}),
\begin{align}
    \lambda_T &= 1\vee \frac{2dM}{S_0}\log\left(e\sqrt{1+\frac{TL}{d}}\vee 1/\delta\right)\label{eq:lambda_t},\\
    \gamma_t(\delta) &= \sqrt{\lambda_T}\left(\frac{1}{2M}+S_0\right) + \frac{4Md}{\sqrt{\lambda_T}} \log\left(e\sqrt{1+\frac{tL}{d} }\vee 1/\delta\right)\spaced{for all} t\in [T]\,.\label{eq:gamma_t}
\end{align}
Here, $S_0=\sup \{ \|\theta\|\,:\, \theta\in \Theta \}$, \todoc{I wonder whether $S_0<\infty$ is really necessary. I am thinking of $Q=\textrm{Exp}(1)$. Then, no problem with unbounded negative values..
Ok to address this after submission.
\\ SL: We still need $\|\theta_\star\|\le S_0$ in the proof of confidence set. } 
as defined in the main body of the paper.
In the algorithm we then choose $\cC_t$ to be $\cC_t^\delta(\hat\theta_t)$ with a fixed value of $\delta\in [0,1]$ that bounds the failure probability of the algorithm.
The following lemma, whose proof is postponed to \cref{section:conf_set_appendix}, as mentioned beforehand, shows that the confidence sets $\cap_{t\ge 1}\cC_t^\delta(\hat\theta_t)$  have coverage $1-\delta$:
%
\begin{restatable}{lemma}{ConfSet}\label{lem:confset} 
Let \cref{ass:model,ass:bdd_var} hold and choose 
$M\ge \max(K/\log(2), 1/(c_1-S_1), 1/(c_2+S_2))$ in \cref{eq:lambda_t,eq:gamma_t},
where $\scfunc$ is any stretch function for $(Q_u)_{u\in [S_2,S_1]}$
and $K=\sup_{S_2\le u\le S_1} \scfunc(u)$.
Then, for the confidence set defined in \cref{eq:conf_set} and 
    for all $\delta\in (0,1]$,
    \[\PP(\forall t\ge 1, \theta_\star\in\cC_t^\delta(\hat\theta_t)) \ge 1-\delta.\]
\end{restatable}
Note that \cref{thm:informal-self-conc} and \cref{ass:bdd_var} guarantees the existence of a stretch function $\scfunc$ mentioned in the theorem. 
\begin{center}
\framebox{In the remainder of this section, we will fix $\scfunc$ to one such stretch function.}
\end{center}
In general, here, one wants to use the smallest such stretch function (i.e., $\scfunc = \scfunc_Q$).
When $\scfunc_Q$ is not available, 
in the lack of a better choice for $\scfunc$, the choice worked out in \cref{thm:informal-self-conc} can be used.

\subsection{Proof of regret upper bound}\label{section:regret_upper_bound}
Let $E_\delta$ be the event that $E_\delta=\{\theta_\star\in \cC_t^\delta(\hat\theta_t)\}$ which by \cref{lem:confset} holds with probability at least $1-\delta$. 
For the next theorem, recall that $S_0=\sup_{\theta\in \Theta}\|\theta\|$ and $S_1=\sup\cU, S_2=\inf\cU$. 

\begin{theorem}\label{thm:regret_bound}
    Let $\delta\in (0,1]$ and $T$ a positive integer
     and consider a well-posed 
    GLB model $\cG = (\cX,\Theta,\cQ)$ 
    and assume that \cref{ass:bdd_var,ass:model} hold.
    Then, by setting $\cC_t = \cC_t^\delta(\hat\theta_t)$, for any $\theta_\star\in \Theta$,
    with probability at least $1-\delta$,
    the regret $\regret(T)$ of \cref{alg:OFU-GLB} when it interacts with the GLB instance
    specified by $\theta_\star$ can be upper bounded by,
    \begin{align*}
        \regret(T)&\le 8c\,\gamma_T(\delta)\sqrt{d\dot\mu(x_\star^\top \theta_\star)(1+L/\lambda)\log\left(1+\frac{LT}{d\lambda}\right)T}\\
        &+ 8c^2\,\gamma^2_T(\delta) L^2K\kappa\log(\lambda+T/d)\\
        &+32c^2\,\gamma^2_T(\delta) \cdot{K d(1+L/\lambda)\log\left(1+\frac{LT}{d\lambda}\right)},
    \end{align*}
    where $c=(1+2K(S_1-S_2))$ and
    \begin{align}
        K=\sup_{S_2\le u\le S_1} \scfunc(u)\,.\label{eq:defn_K}
    \end{align}
\end{theorem}
\begin{proof}
We first consider the case that the base distribution has $0$ variance, which implies that $Q$ is a Dirac.
As discussed beforehand, and as it is easy to see it, in this case $\cU_Q=\R$, $Q_u = Q$ for any $u\in \R$. \todoc{Could make this into a proposition. But then need to use this proposition every time the argument is used. After submission.}
Hence, all arms have the same payoff and all algorithm incur zero regret. 
In the rest of this proof, we assume that $\mathrm{Var}(Q)>0$. 
Since $\mu(\cdot)=\dot\psi_Q(\cdot)$ is infinitely differentiable (\cref{prop:moments_of_NEF}), we can perform a second-order Taylor expansion on the regret
\begin{align*}
    \regret(T) &= \sum_{t=1}^T \mu(x_\star ^\top \theta_\star) - \mu(X_t^\top\theta_\star)\\
    &=\underbrace{\sum_{t=1}^T \dot\mu(X_t^\top\theta_\star)(x_\star - X_t)^\top\theta_\star}_{R_1(T)} + \underbrace{\frac{1}{2}\sum_{t=1}^T \ddot\mu(\xi_t)((x_\star - X_t)^\top\theta_\star)^2}_{R_2(T)},
\end{align*}
    where $\xi_t$ is between $X_t^\top\theta_\star$ and $x_\star^\top\theta_\star$ for all $t\in [T]$. On event $E_\delta$, by definition of $X_t, \theta_t$ (in \cref{alg:OFU-GLB}), it holds that 
    $x_\star^\top \theta_\star\le X_t^\top\theta_t$. 
    Observe that $\gamma_t(\delta)$ (\cref{eq:gamma_t}) is increasing in $t$, we have that $\gamma_t(\delta)\le \gamma_T(\delta)$ for all $t\in [T]$. Then we can bound $R_1(T)$ as follows
    \begin{align*}
        R_1(T)&=\sum_{t=1}^T \dot\mu(X_t^\top\theta_\star)(x_\star - X_t)^\top\theta_\star\\
        &\le \sum_{t=1}^T \dot\mu(X_t^\top\theta_\star)X_t^\top (\theta_t-\theta_\star)\\
        &\le \sum_{t=1}^T \dot\mu(X_t^\top\theta_\star)\|X_t\|_{H_t^{-1}(\theta_\star)}\|\theta_t-\theta_\star\|_{H_t(\theta_\star)}
    \end{align*}
    where the last inequality is due to Cauchy-Schwarz. Since $\theta_t$ and $\theta_\star$ are all in the confidence set on $E_\delta$, we are able to bound $\|\theta_t-\theta_\star\|_{H_t(\theta_\star)}$ by the following lemma that exploits the properties of confidence set as well as self-concordant functions. This lemma is a variation of proposition 4 of \citet{Abeille2020InstanceWiseMA} where they show the result for logistic function that is $1$-self-concordant. 
\begin{restatable}[$\cC_t^\delta(\hat\theta_t)$ has small ellipsoidal diameters] {lemma}{ConfSetDiamH}\label{lem:conf_set_diam_H}
    Under \cref{ass:bdd_var,ass:model}, for all $\theta_1,\theta_2\in \cC_t^\delta(\hat\theta_t)$, it follows that 
    \begin{align*}
        \|\theta_1-\theta_2\|_{H_t(\theta_1)}\vee  \|\theta_1-\theta_2\|_{H_t(\theta_2)}\le 2(1+2K\cdot (S_1-S_2))\gamma_t(\delta),
    \end{align*}
    where $K$ is defined in \cref{eq:defn_K}.
\end{restatable}
By \cref{lem:conf_set_diam_H}, we can upper bound $R_1(T)$ to be
\begin{align*}
    R_1(T)\le \sum_{t=1}^T \dot\mu(X_t^\top\theta_\star)\|X_t\|_{H_t^{-1}(\theta_\star)} \cdot 2(1+2K(S_1-S_2))\gamma_T(\delta).
\end{align*}
Denote $A_t=\sqrt{\dot\mu(X_t^\top\theta_\star)}X_t$ and we have that $H_t(\theta_\star)= \sum_{s=1}^t A_tA_t^\top+\lambda I$ as well as $\|A_t\|\le \sqrt{L}\le L$ where the second inequality is because WLOG we assume $L\ge 1$ in \cref{ass:bdd_var}. We can bound $R_1(T)$ in the terms of $A_t$.
    \begin{align*}
        R_1(T) &\le 2(1+2K(S_1-S_2))\gamma_T(\delta)\sum_{t=1}^T \sqrt{\dot\mu(X_t^\top\theta_\star)\|X_t\|_{H_t^{-1}(\theta_\star)}^2}\sqrt{\dot\mu(X_t^\top\theta_\star)}\\
        &\le 2(1+2K(S_1-S_2))\gamma_T(\delta)\sqrt{\sum_{t=1}^T\dot\mu(X_t^\top\theta_\star)\|X_t\|_{H_t^{-1}(\theta_\star)}^2}\sqrt{\sum_{t=1}^T \dot\mu(X_t^\top\theta_\star)}\tag{Cauchy-Schwarz}\\
        &= 2(1+2K (S_1-S_2))\gamma_T(\delta)\sqrt{\sum_{t=1}^T \|A_t\|^2_{H_t^{-1}(\theta_\star)}}\sqrt{\sum_{t=1}^T \dot\mu(X_t^\top\theta_\star)}\\
        &\le 4(1+2K(S_1-S_2))\gamma_T(\delta)\sqrt{d(1+L/\lambda)\log\left(1+\frac{LT}{d\lambda}\right)}\sqrt{\sum_{t=1}^T \dot\mu(X_t^\top\theta_\star)}
    \end{align*}
     where in the last step we use elliptical potential lemma of \citet{abbasi2011improved}, which, for easy of reference, we give in \cref{lem:epl}.
    We now start to bound $R_2(T)$. For convenience, we throw away the factor of $1/2$.
    \begin{align*}
        R_2(T)&\le \sum_{t=1}^T \ddot\mu(\xi_t)((x_\star - X_t)^\top\theta_\star)^2\\
        &\le \sum_{t=1}^T \ddot\mu(\xi_t)(X_t^\top(\theta_t-\theta_\star))^2\tag{$X_t^\top\theta_t\ge x_\star^\top\theta_\star\ge X_t^\top\theta_\star$}\\
        &\le \sum_{t=1}^T\ddot\mu(\xi_t)\norm{X_t}_{H_t^{-1}(\theta_\star)}^2\|\theta_t-\theta_\star\|_{H_t(\theta_\star)}^2\\
        &\le 4(1+2K (S_1-S_2))^2\gamma_T(\delta)^2\sum_{t=1}^T\ddot\mu(\xi_t)\norm{X_t}_{H_t^{-1}(\theta_\star)}^2
    \end{align*}
    where in the last inequality we use \cref{lem:conf_set_diam_H}. By definition of self-concordant function, we have that $\ddot\mu(\xi_t)\le \scfunc(\xi_t)\dot\mu(\xi_t)$. Since $\xi_t$ is between $x_\star^\top\theta_\star$ and $X_t^\top\theta_\star\,$. Note that $\scfunc$ defined in \cref{thm:informal-self-conc} is increasing on $[0,c_1)$ and decreasing on $(-c_2,0)$, which gives us $\scfunc(\xi_t)\le\scfunc(X_t^\top\theta_\star)\vee \scfunc(x_\star^\top\theta_\star) \le K$. 
Let $V_t=\lambda I + \sum_{i=1}^t X_iX_i^\top$.
    We hence have
        \begin{align*}
        R_2(T)&\le 4(1+2K (S_1-S_2))^2\gamma_T(\delta)^2\sum_{t=1}^T K\dot\mu(\xi_t)\norm{X_t}_{H_t^{-1}(\theta_\star)}^2\\
        &\le 4(1+2K(S_1-S_2))^2\gamma_T(\delta)^2KL\sum_{t=1}^T\norm{X_t}_{H_t^{-1}(\theta_\star)}^2\tag{$\dot\mu(\cdot)\le L$ (\cref{ass:bdd_var})}\\
        &\le 4(1+2K(S_1-S_2))^2\gamma_T(\delta)^2KL\kappa\cdot\sum_{t=1}^T\norm{X_t}_{V_t^{-1}}^2\tag{$H_t^{-1}(\theta_\star)\preceq \kappa V_t^{-1}$}\\
        &\le 4(1+2K(S_1-S_2))^2\gamma_T(\delta)^2KL\kappa\cdot L\log(\lambda+T/d)\tag{\cref{lem:epl}}.
    \end{align*}
    Putting the bound on $R_1(T)$ and $R_2(T)$ together, we have that
    \begin{align}
        \regret(T)&=R_1(T)+R_2(T)\notag\\
        &\le 4(1+2K(S_1-S_2))\gamma_T(\delta)\sqrt{d(1+L/\lambda)\log\left(1+\frac{LT}{d\lambda}\right)}\sqrt{\sum_{t=1}^T \dot\mu(X_t^\top\theta_\star)}\notag\\
        &+4(1+2K(S_1-S_2))^2\gamma_T(\delta)^2L^2K\kappa\log(\lambda+T/d).\label{eq:before_dot_mu_sum}
    \end{align}
    We mimic the trick used in \citet{janz2023exploration} to bound $\sqrt{\sum_{t=1}^T \dot\mu(X_t^\top\theta_\star)}$ which was originally proposed by \citet{Abeille2020InstanceWiseMA}. We present the following lemma that is abstracted out from Claim 14 of \citet{janz2023exploration}.
    \begin{restatable}[Self-bounding property of self-concordance functions]{lemma}{DotMuSum}\label{lem:dot-mu-sum}
Let $\cV=[a,b]$, a closed, nonempty interval over the reals,
$f$ a real valued function defined over an interval of the reals 
that is twice continuously differentiable over $\cV$
such that for some $\selfconc:\cV\to \R_+$, $|\ddot f(v)|\le \selfconc(v) \dot f(v)$ holds for all $v\in \cV$. Assume that $A=\sup_{v\in \cV} \scfunc(v)<\infty$. 
Furthermore, 
assume that either $\dot f$ is identically zero over $\cV$, or $\dot f$ is positive valued over $\cV$.
    For $n$ a positive integer, 
    let $\{a_t\}_{t=1}^n\subset \cV$. Then, 
    \begin{equation*}
        \sum_{t=1}^n \dot f(a_t) \le n \dot f(b) + A\sum_{t=1}^n f(b)- f(a_t).
    \end{equation*}
\end{restatable}
We apply this lemma with $f=\mu$, $[a,b]=[S_2,x_\star^\top\theta_\star]\subset [S_2,S_1]$ and $\scfunc$ restricted to $[a,b]$. Then, all the conditions of the lemma are satisfied by our choice of $\scfunc$. Furthermore,
 $A=\sup_{v\in [S_2,x_\star^\top\theta_\star]}\scfunc(v)\le K<+\infty$.  
 Hence, all the conditions of the lemma are verified. Hence, 
 \begin{align}
    \sqrt{\sum_{t=1}^T\dot\mu(X_t^\top\theta_\star)}&\le \sqrt{T\dot\mu(x_\star^\top\theta_\star)+K\regret(T)}\notag\\
    &\le \sqrt{T\dot\mu(x_\star^\top\theta_\star)}+\sqrt{K\regret(T)}\label{eq:dot_mu_sum}.
\end{align}
Plug \cref{eq:dot_mu_sum} into \cref{eq:before_dot_mu_sum},
    \begin{align*}
        \regret(T)&\le  4(1+2K(S_1-S_2))\gamma_T(\delta)\sqrt{d\dot\mu(x_\star^\top \theta_\star)(1+L/\lambda)\log\left(1+\frac{LT}{d\lambda}\right){T }}\\
        &+4(1+2K(S_1-S_2))^2\gamma_T(\delta)^2L^2K\kappa\log(\lambda+T/d)\\
        &+4(1+2K(S_1-S_2))\gamma_T(\delta)\sqrt{Kd(1+L/\lambda)\log\left(1+\frac{LT}{d\lambda}\right)}\sqrt{\regret(T)}.
    \end{align*}
    Let
    \begin{align*}
        A&=4(1+2K(S_1-S_2))\gamma_T(\delta)\sqrt{K d(1+L/\lambda)\log\left(1+\frac{LT}{d\lambda}\right)},\\
        B&=4(1+2K(S_1-S_2))\gamma_T(\delta)\sqrt{d\dot\mu(x_\star^\top \theta_\star)(1+L/\lambda)\log\left(1+\frac{LT}{d\lambda}\right)T}\\
        &+4(1+2K(S_1-S_2))^2\gamma_T(\delta)^2L^2K\kappa\log(\lambda+T/d),
    \end{align*}
    we can write out the inequality
    \begin{equation*}
        \regret(T)\le A\sqrt{\regret(T)}+B.
    \end{equation*}
    Solving it we have that 
    \begin{equation*}
        \regret(T)\le 2A^2+2B.
    \end{equation*}
    Plugging in the definition of $A$ and $B$ back, 
    \begin{align*}
        \regret(T)&\le 8(1+2K(S_1-S_2))\gamma_T(\delta)\sqrt{d\dot\mu(x_\star^\top \theta_\star)(1+L/\lambda)\log\left(1+\frac{LT}{d\lambda}\right)T}\\
        &+ 8(1+2K(S_1-S_2))^2\gamma_T(\delta)^2L^2K\kappa\log(\lambda+T/d)\\
        &+32(1+2K(S_1-S_2))^2\gamma_T(\delta)^2\cdot{K d(1+L/\lambda)\log\left(1+\frac{LT}{d\lambda}\right)}.\qedhere
    \end{align*}
\end{proof}

\subsection{Self-concordance control}\label{section:self-conc-control}
In this section we provide technical results about self-concordant functions which play important roles in confidence set construction and controlling the regret of \cref{alg:OFU-GLB}. Specifically, \cref{lem:dot_mu_rel} and \cref{lem:alpha_dotmu_relate} are used to show \cref{lem:conf_set_diam_H}, one of the key lemmas we use in the proof of \cref{thm:regret_bound}. \cref{lem:tiltsubexp} is used to justify the confidence set contains $\theta_\star$ with high probability (\cref{lem:confset}). 

The next lemma shows that for self-concordant NEFs, $\dot \mu$ is a smooth function of its argument.\todos{specify exactly which lemma is used where}
The lemma is essentially the same as Lemma 3 from \citet{janz2023exploration} (itself based on a result of  \citet{sun2019generalized}) and is updated only to match our definitions of self-concordance, which is a refinement of that used by  \citet{janz2023exploration}.  The proof (based on the proof of a similar result of  \citet{sun2019generalized}) is included for the convenience of the reader.
\begin{lemma}[Self-concordance to smoothness]\label{lem:gdot_mu_rel}
Let $\cU$ be an interval over the reals,
$\mu$ a real valued function defined over an interval of the  reals 
that is twice continuously differentiable over $\cU$
such that for some $\selfconc:\cU\to \R_+$, $|\ddot\mu(u)|\le \selfconc(u) \dot\mu(u)$ holds for all $u\in \cU$. Assume that $K=\sup_{u\in \cU} \scfunc(u)<\infty$.  
Assume that either $\dot\mu$ is identically zero over $\cU$, or $\dot\mu$ is positive valued over $\cU$.
Then, for any $u,u'\in \cU$,
    \begin{equation*}
        \dot\mu(u')\le \dot\mu(u)e^{K|u-u'|}.
    \end{equation*}
\end{lemma}
An immediate corollary of this lemma is that self-concordance of a NEF implies that the variance function, $\dot\mu$, of the NEF is smooth:
\begin{corollary}[Self-concordance to smoothness]\label{lem:dot_mu_rel}
Let $\cQ= (Q_u)_{u\in \cU}$ be self-concordant with stretch function $\selfconc:\cU \to \R_+$,
where $\cU$ is an interval 
and assume $K=\sup_{u\in \cU} \scfunc(u)<\infty$.  
Then for any $u,u'\in \cU$,
    \begin{equation*}
        \dot\mu(u')\le \dot\mu(u)e^{K|u-u'|}.
    \end{equation*}
\end{corollary}
Note that the inequality is well-posed since $u,u'\in \cU_Q^\circ$, and $\mu$ is known to be differentiable over $\cU_Q^\circ$ and, by the definition of self-concordance, $\cU\subset \cU_Q^\circ$.
\begin{proof}
This result follows from \cref{lem:gdot_mu_rel} once we notice that the variance function of a NEF is such that if $\dot\mu(u)=0$ for any $u\in\cU$, then $\dot\mu$ is identically zero over $\cU$. Indeed, if $\dot\mu(u)=0$, then $Q_u$ is a Dirac distribution and so is $Q_v$ for any $v\in \cU$.
\end{proof}
\begin{proof}[Proof of \cref{lem:gdot_mu_rel}]
When $\dot\mu$ is identically zero over $\cU$, the statement is trivial.
Hence, consider now the case when $\dot\mu$ is positive valued over $\cU$:
\begin{align}
\dot\mu(v)>0 \qquad \text{for all }\quad v\in \cU\,. \label{eq:dmupos}
\end{align}
Then,
it suffices to show that $\ln \frac{\dot\mu(u')}{\dot\mu(u)}\le K|u-u'|$.
To show this, define $\phi(t) = \dot\mu(u+t(u'-u))$ so that
$\phi(0)=\dot\mu(u)$ and
 $\phi(1)=\dot\mu(u')$. 
 Since $\cU_Q$ is an interval with non-empty interior, $\phi(t)$ is well-defined for all $t\in [0,1]$. 
 Furthermore, by \cref{eq:dmupos} and since $\cU$ is an interval,
    we have that $\phi(t)>0$ for all $t\in [0,1]$.
Consider now the map     
$t\mapsto \ln\phi(t)$ where $t\in [0,1]$.
The derivative of this map exist and is continuous over $(0,1)$, and in particular, $\frac{d}{dt} \ln \phi(t) = \frac{\dot\phi(t)}{\phi(t)}$ by the chain rule. Indeed, the derivative of $\phi$ exists and is continuous over $(0,1)$, because the same holds for $\dot\mu$ by the properties  of NEFs, and as we just discussed, $\phi$ is positive over $[0,1]$ and is continuous.
 Now, by the fundamental theorem of calculus applied to $t\mapsto \frac{d}{dt} \ln \phi(t)$
 and by the monotonicity of integrals,
    \begin{equation}
      \ln \frac{\dot\mu(u')}{\dot\mu(u)}=
        \ln \frac{\phi(1)}{\phi(0)} =  \int_{0}^1 \frac{d\ln \phi(t)}{dt}dt\le \int_{0}^1 \left|\frac{d\ln \phi(t)}{dt}\right|dt\label{eq:log_ratio_integral}.
    \end{equation}
It remains to bound the integrand in the rightmost expression.
For this, as discussed earlier we have
    \begin{align}
        \left|\frac{d\ln \phi(t)}{dt}\right|=\left|\frac{\phi'(t)}{\phi(t)}\right|=\frac{|\phi'(t)|}{\phi(t)}\,.\label{eq:log_ratio_integrand}
    \end{align}
To bound the ratio on the right, we again use the chain rule and calculate
    \begin{equation*}
        |\phi'(t)| = |\ddot\mu(u+t(u'-u))||u'-u|
        \le K\underbrace{\dot\mu(u+t(u'-u))}_{\phi(t)}|u'-u|\,,
    \end{equation*}
where the inequality follows by the definition of $K$ by definition of self-concordant function,
we have that for all $u\in \cU$, $|\ddot\mu(u)|\le K\dot\mu(u)$. 
Now, the result follows since we have shown that the integrand is upper bounded by $K|u-u'|$ and thus
$\int_0^1 \left|\frac{d\ln \phi(t)}{dt}\right|dt \le K |u-u'|$.
\end{proof}
We continue with two results, both of which use the lemma just proved. The first result gives a lower bound for the integral remainder term when Taylor's theorem is used to approximate $\mu$. The second result gives a quadratic upper bound on the \CGF of $Q_u$, and will be the basis for constructing our confidence set. \todoc{where is this lemma from?!}
\begin{lemma}\label{lem:alpha_dotmu_relate}
Let $\cQ= (Q_u)_{u\in \cU}$ be self-concordant with stretch function $\selfconc:\cU\to \R_+$ where $\cU$ is an interval,
and assume $K=\sup_{u\in \cU} \scfunc(u)<\infty$.  
Then for any $u,u'\in \cU$,
    \begin{equation*}
        \int_0^1 \dot\mu(u+t(u'-u))dt \ge 
        \frac{\dot\mu(u)}{1+K|u-u'|}\,.
    \end{equation*}
\end{lemma}
\begin{proof}
    By \cref{lem:dot_mu_rel}, it follows that 
    \begin{equation*}
        \dot\mu(u+t(u'-u))\ge \dot\mu(u)\exp(-Kt|u'-u|).
    \end{equation*}
    Integrating both sides between $0$ and $1$ gives us
    \begin{align*}
    \MoveEqLeft
        \int_0^1 \dot\mu(u+t(u'-u)) 
        \ge \dot\mu(u)\int_0^1 \exp(-Kt |u'-u|)dv\\
        &= \dot\mu(u)\left[ \frac{-\exp(-Kt |u'-u|)}{K|u'-u|} \right]_0^1\\
        &=\dot\mu(u) \frac{1-\exp(-K|u'-u|) }{K|u'-u|}\\
        &\ge \frac{\dot\mu(u)}{1+K|u'-u|}\,,
    \end{align*}
    where the last inequality follows from the elementary inequality  $(1-e^{-x})/x\ge 1/(1+x)$ that holds for all $x\ge 0$.
\end{proof}
We are now ready to prove \cref{lem:tiltsubexp}. As noted beforehand, we adopt this lemma 
from the work of \citet{janz2023exploration}. In particular, it is an adaptation of their 
 Lemma~1, which was proved
for distributions where $\cU_Q = \RR$.
Here, we deal with the case when $\cU_Q$ is possibly a strict subset of $\RR$.
%
\SCtoTail*

\begin{proof}
Let $u\in \cU$. Hence, by our assumption on $\cU$, $u\in \cU_Q$.
Now let $s\in \R$. Then,
    \begin{align*}
        \psi_{Q_u}(s)&=\log \int \exp(sy)Q_u(dy)
        =\log\left[\frac{1}{M_Q(u)}\int \exp(sy)\exp(uy)Q(dy)\right]\\
        &=\psi_Q(u+s) - \psi_Q(u)\,.
    \end{align*}
    Hence, $\psi_{Q_u}(s)$ is finite valued whenever $u+s\in \cU_Q$.
    Assume that this holds and in fact $u+s\in \cU_Q^\circ$.
    
    Since $u,u+s\in \cU_Q^\circ$, 
    $\psi_Q$ is twice continuously differentiable over an open interval containing $u$ and $u+s$.
    Then, by Taylor's theorem there exists $\xi$ in the closed interval between $u$ and $u+s$ 
    such that
    \begin{align*}
        \psi_Q(u+s) - \psi_Q(u) &= s\dot\psi_Q(u)+\frac{s^2}{2}\ddot\psi_Q(\xi)\,.
    \end{align*}
   	Since $\dot\psi_Q = \mu$ and $\ddot \psi_Q = \dot\mu$ (cf. 
    \cref{prop:moments_of_NEF}) we get
    \begin{align*}
        \psi_Q(u+s) - \psi_Q(u) &=s\mu(u)+\frac{s^2}{2}\dot\mu(\xi)\,.
	\end{align*}
	Now, by \cref{lem:dot_mu_rel}, we have that
    \begin{align*}
        \dot\mu(\xi) &\le  \dot\mu(u) \cdot e^{K|u-\xi|}\le  \dot\mu(u)e^{Ks}\le 2\dot\mu(u),
    \end{align*}
    where the final inequality follows when $|s| \le \log(2)/K$.
	Putting things together, it follows that if $|s|\le \log(2)/K$ and $s\in \cU_Q^\circ -\{u\}$ then
    \begin{equation*}
        \psi_{Q_u}(s) = \psi_Q(u+s) - \psi_Q(u) \le s\mu(u)+ s^2\dot\mu(u)\,.
    \end{equation*}
    For $S\subset \R$, $r\in \R$, let $S\pm r = \{ s\pm r\,:\, s\in S\}$.
    Since $u$ is an arbitrary point in $\cU$, the above conditions on $s$ will be satisfied 
    if $|s|\le \log(2)/K$ and $s\in Z\doteq \cap_{u\in \cU} \cU_Q^\circ -u$.
    Now, from $\cU_Q^\circ = (a,b)$, we have 
    $Z 
    =\cap_{u\in \cU} (a-u,b-u) = (\sup_{u\in \cU} a-u, \inf_{u\in \cU} b-u ) = (a-\inf \cU, b-\sup\cU)$, 
    finishing the proof.
\end{proof}
From the calculation at the end of the proof it follows that the statement of the lemma is non-vacuous if 
for $\cU_Q^\circ = (a,b)$, $a-\inf \cU<0<b-\sup\cU$, which is equivalent to that 
$a<\inf\cU$ and $\sup\cU<b$, which is always satisfied when $\cU$ is a strict subset of $\cU_Q^\circ$.
\subsection{Confidence set construction}
\label{section:conf_set_appendix}

We now turn to proving  \cref{lem:confset} which is concerned with showing that the confidence sets $\cC_t^\delta$ contain the true parameter $\theta_\star$ with probability $1-\delta$.
The proof is essentially the same as that of Lemma 4 of \cite{janz2023exploration}. 
We will need the following result, which is taken verbatim from the paper of  \citet{janz2023exploration}. 
\begin{proposition}[Theorem 2 of \cite{janz2023exploration}]\label{thm:self_norm_martingale}
    Fix $\lambda,M >0$. Let $(X_t)_{t\in \NN^+}$ be a $B_2^d$-valued random sequence, $(Y_t)_{t\in \NN^+}$ a real valued random sequence and $(\nu_t)_{t\in \NN}$ be a nonnegative valued random sequence. Let $\FF=(\FF_t)_{t\in \NN}$ be a filtration such that $(i)$ $(X_t)_{t\in \NN^+}$ is $\FF$-predictable and $(ii)$ $(Y_t)_{t\in \NN^+}$ are $\FF$-adapted. Let $\epsilon_t=Y_t-\E[Y_t\vert \FF_{t-1}]$ and assume that the following condition holds:
    \begin{equation*}
        \E[\exp(s\epsilon_t)\vert \FF_{t-1}]\le \exp(s^2{\nu_{t-1}}) \spaced{for all} |s|\le 1/M \text{ and }t\in \NN^+.
        \end{equation*}
        Then, for $\tilde H_t=\sum_{i=1}^t \nu_{i-1}X_iX_i^\top +\lambda I$ and $\bS_t=\sum_{i=1}^t\epsilon_iX_i$ and $\delta>0$,
        \begin{equation*}
            \PP\left(\exists t\in \NN^+:\|\bS_t\|_{\tilde H_t^{-1}}\ge \frac{\sqrt{\lambda}}{2M}+\frac{2M}{\sqrt{\lambda}}\log\left(\frac{\det (\tilde H_t)^{1/2}\lambda^{-d/2}}{\delta}\right)+\frac{2M}{\sqrt{\lambda}}d\log (2)\right)\le \delta.
        \end{equation*}
\end{proposition}

We now turn to proving \cref{lem:confset}.
For the convenience of the reader, we start by  recalling the definition of the confidence sets $\cC_t^\delta$ involved (cf.  \cref{eq:conf_set}). Recall that $c_2<S_2\le x^\top\theta \le S_1<c_1$ for $x\in \cX$ and $\theta\in \Theta$. For $\delta\in (0,1]$, we have
\begin{equation*}
    \cC_t^{\delta,\lambda}(\hat\theta_t)=\left\{\theta\in \Theta\,:\,\left\|g_t(\theta)-g_t(\hat\theta_t)\right\|_{H_t^{-1}(\theta)}\le \gamma_t(\delta)\right\}
\end{equation*}
where
\begin{align*}
    \gamma_t(\delta) &= \sqrt{\lambda_T}\left(\frac{1}{2M}+S_0\right) + \frac{4Md}{\sqrt{\lambda_T}} \log\left(e\sqrt{1+\frac{tL}{d} }\vee 1/\delta\right)\spaced{for all} t\in [T]\,,\\
    \lambda_T &= 1\vee \frac{2dM}{S_0}\log\left(e\sqrt{1+\frac{TL}{d}}\vee 1/\delta\right),
\end{align*}
and recall that $S_0= \sup_{\theta\in \Theta} \| \theta \|$ or an upper bound on this quantity, and $M$ is specified in the next result:
\ConfSet*
\begin{proof}
    From definition it follows that $\theta_\star \in \Theta$. Now we prove that with probability at least $1-\delta$, it holds that $\|g_t(\theta_\star) - g_t(\hat\theta_t)\|\le \gamma_T(\delta)$ for all $t\ge 1$. The proof goes through by using \cref{thm:self_norm_martingale} and we now match the conditions of \cref{thm:self_norm_martingale}.
    By \cref{ass:model}, we have that $(X_t)_{t\in \NN^+}$ is a $B_2^d$-valued random sequence. 
    Let $\cF_{t-1}=\sigma(X_1,Y_1,...,X_{t-1},Y_{t-1},X_t)$ for $t\ge 1$. Consider the filtration $\cF=(\cF_t)_{t\in \NN}$. Then by definition, $(X_t)_{t\in \NN^+}$ are $\cF$-predictable and $(Y_t)_{t\in \NN^+}$ are $\cF$-adapted. Note that $\mu(X_i^\top\theta_\star) = \EE[Y_i|\cF_{i-1}]$ for all $i\in [n]$. Let $\varepsilon_i = Y_i-\EE[Y_i|\cF_{i-1}]$. This gives $(\varepsilon_t)_{t\in \NN^+}$ are also $\cF$-adapted and the following identity follows by definition
    \begin{align*}
        g_t(\hat\theta_t) - g_t(\theta_\star)&=\sum_{i=1}^t \varepsilon_iX_i+\lambda\theta_\star\\
        &=\bS_t+\lambda \theta_\star,
    \end{align*}
    where $\bS_t=\sum_{i=1}^t \varepsilon_iX_i$. 
    Let
    \begin{equation*}
        \nu_{t-1} = \dot\mu(X_t^\top\theta_\star).
    \end{equation*}
    Now we would like to apply \cref{lem:tiltsubexp} to show that, for $|s|\le M$, it follows that 
    \begin{equation}
        \EE[\exp(s\varepsilon_t)| \cF_{t-1}]\le \exp(s^2\nu_{t-1}).\label{eq:apply_tildsubexp}
    \end{equation}
    Applying the definition of $\varepsilon_t$, it follows that
    \begin{align*}
        \EE[\exp(s\varepsilon_t)| \cF_{t-1}]&=\EE[\exp(sY_t - s\mu(X_t^\top\theta_\star)| \cF_{t-1}]\\
        &=\exp(-s\mu(X_t^\top\theta_\star)) \EE[Y_t|\cF_{t-1}] \quad\mathrm{a.s.}\,,
    \end{align*}
    where the last equality is because $X_t$ is $\cF_{t-1}$-measurable.
    Since, by definition, the distribution of $Y_t$ given $\cF_{t-1}$ is
    $Q_{X_t^\top\theta_\star}$, we have that 
    \begin{align*}
        \EE[\exp(s\varepsilon_t)|\cF_{t-1}]
        &= \exp(-s\mu(X_t^\top \theta_\star) ) \EE[\exp(s Y_t) |\cF_{t-1}]\\
        &=\exp(-s\mu(X_t^\top \theta_\star) ) \int_{\RR}e^{sy}Q_{X_t^\top\theta_\star}(dy)\\
        &=\exp(-s\mu(X_t^\top \theta_\star) + \psi_{Q_{X_t^\top\theta_\star}}(s))\\
        &\le \exp(s^2\dot\mu(X_t^\top\theta_\star)),
    \end{align*}
    where the last inequality is because $|s|\le M$ implies that 
    $s\in [-\log 2/K,\log 2/K]\cap (\cU_Q^\circ - S_2)\cap (\cU_Q^\circ - S_1)$ 
    so \cref{lem:tiltsubexp} is applicable 
    (it is applied with $(Q_u)_{u\in [S_2,S_1]}$ and $\scfunc$ as chosen in the statement).
Then,  
    \cref{eq:apply_tildsubexp} follows by noting that $\nu_{t-1}=\dot\mu(X_t^\top\theta_\star)$.
    
    Lastly as defined above, $\tilde H_t$ corresponds to $H_t(\theta_\star)$.
    Taking the $\ell_2$-norm weighted by $H_t^{-1}(\theta_\star)$ and applying triangle inequality, \todoc{So this is where $S_0$ still appears. If we replaced the regularizer with $\lambda \| \theta \|_G^2$ with an appropriately selected $G$ matrix (i.e., $G = \sum_{i=1}^n \nu_i x_i x_i^\top$ where $\sum_i \nu_i=1$, $\nu_i\ge 0$, $x_i\in \cX$), use optimal design?, I bet $S_0$ will go away.
    At minimum we can ignore all the components of $\theta$ that are orthogonal to the subspace spanned by $\cX$ in $\R^d$.}
    \begin{align*}
        \|g_t(\hat\theta_t) - g_t(\theta_\star)\|_{H_t^{-1}(\theta_\star)}\le \|\bS_t\|_{H_t^{-1}(\theta_\star)}+\lambda \|\theta_\star\|_{H_t^{-1}(\theta_\star)}\le \|\bS_t\|_{H_t^{-1}(\theta_\star)}+\sqrt \lambda S_0,
    \end{align*}
    where the last inequality follows by $H_t^{-1}(\theta_\star)\preceq \lambda^{-1} I$.
    By \cref{thm:self_norm_martingale}, with probability at least $1-\delta$, it follows that for all $t\ge 1$,
    \begin{equation*}
        \|\bS_t\|_{H_t^{-1}(\theta_\star)}< \frac{\sqrt{\lambda}}{2M}+\frac{2M}{\sqrt{\lambda}}\log\left(\frac{\det (H_t(\theta_\star))^{1/2}/\lambda^{d/2}}{\delta}\right)+\frac{2M}{\sqrt{\lambda}}d\log (2).
    \end{equation*}
    We now bound $\det (H_t(\theta_\star))/\lambda^d$. Let $A_i=\sqrt{\dot\mu(X_i^\top\theta_\star)}X_i$ for all $i\in [t]$, then $H_t(\theta_\star)$ can be written as
    \begin{equation*}
        H_t(\theta_\star) = \lambda I + \sum_{s=1}^t A_iA_i^\top.
    \end{equation*}
    By \cref{ass:bdd_var}, it holds that $\sqrt {\dot\mu(X_t^\top\theta_\star)}\le \sqrt L$,
    thus $\|A_i\|_2\le \sqrt{L}\le L$ for all $i\in [t]$.
    Eq. (20.9) (Note 1 of section 20.2) in \cite{Lattimore_Szepesvári_2020} gives 
    \begin{equation*}
        \det (H_t(\theta_\star))/\lambda^d \le \left(1+\frac{t L}{\lambda d}\right)^d.
    \end{equation*}
    The stated result follows by chaining all the inequalities together and noting that 
    \begin{align*}
    \gamma_t(\delta)\ge  \sqrt{\lambda}\left(\frac{1}{2M}+S_0\right)+\frac{2Md}{\sqrt \lambda}\left(1+\frac{1}{2}\log\left(1+\frac{t L}{\lambda d}\right)\right)+\frac{2M}{\sqrt\lambda}\log(1/\delta),\spaced{for all} t\in [T].
\end{align*}
    
\end{proof}
\subsection{Proof of \cref{lem:conf_set_diam_H}}\label{section:proof_of_conf_set_diam_H}

The following two lemmas (\cref{lem:MVT_application,lem:GtHt_relate}) are variations of Claim 4 and Claim 3 of \citet{janz2023exploration}. The difference is that \citet{janz2023exploration} show them for all $\theta_1, \theta_2\in \RR^d$ because the MGF $M_Q$ therein is finite on $\RR$. In our setting, there could be $x\in \cX$ for some $\theta\notin \Theta$ such that $M_Q(x^\top\theta)=\infty$, hence we show it within the parameter set $\Theta$.
\begin{lemma}\label{lem:MVT_application}
    For all $\theta_1,\theta_2\in \Theta$, it follows that
    \begin{equation*}
        g_t(\theta_1) - g_t(\theta_2) = G_t(\theta_1,\theta_2)(\theta_1-\theta_2).
    \end{equation*}
    In particular, we have that
    \begin{equation*}
        \|g_t(\theta_1) - g_t(\theta_2) \|_{G_t^{-1}(\theta_1,\theta_2)} = \|\theta_1-\theta_2\|_{G_t(\theta_1,\theta_2)}.
    \end{equation*}
\end{lemma}
\begin{proof}
The ``In particular'' part follows from definition of $\ell_2$-norm weighted by $G_t^{-1}(\theta_1,\theta_2)$. We now prove $g_t(\theta_1)-g_t(\theta_2) = G_t(\theta_1,\theta_2)(\theta_1-\theta_2)$.
    By definition of the difference quotient $\alpha(\cdot, \cdot)$, we have that
    \begin{equation}
        \mu(u) - \mu(u')=\alpha(u,u')(u-u').\label{eq:mu_alpha_relate}
    \end{equation}
    Writing out the expression of $g_t(\theta_1)-g_t(\theta_2)$ gives
    \begin{align*}
        g_t(\theta_1)-g_t(\theta_2) &= \sum_{i=1}^t \Big(\mu(X_i^\top\theta_1) - \mu(X_i^\top\theta_2)\Big)X_i+\lambda (\theta_1-\theta_2)\\
        &=\sum_{i=1}^t \Big(\alpha(X_i^\top\theta_1,X_i^\top\theta_2)X_i^\top(\theta_1-\theta_2)\Big)X_i+\lambda(\theta_1-\theta_2)\tag{\cref{eq:mu_alpha_relate}}\\
        &=\left(\sum_{i=1}^t\alpha(X_i^\top\theta_1,X_i^\top\theta_2)X_iX_i^\top\right)(\theta_1-\theta_2)+\lambda(\theta_1-\theta_2)\\
        &=G_t(\theta_1,\theta_2)(\theta_1-\theta_2). \qedhere
    \end{align*}
\end{proof}

\begin{lemma}\label{lem:GtHt_relate}
   Under \cref{ass:bdd_var,ass:model}, for all $\theta_1,\theta_2\in \Theta$, it follows that
    \begin{align}
        G_t(\theta_1,\theta_2)&\succeq (1+2K\cdot (S_1-S_2))^{-1}H_t(\theta_1)\label{eq:GtHt_relate_1}\\
        G_t(\theta_1,\theta_2)&\succeq (1+2K\cdot (S_1-S_2))^{-1}H_t(\theta_2)\label{eq:GtHt_relate_2},
    \end{align}
    where $K$ is defined in \cref{eq:defn_K}. 
\end{lemma}
\begin{proof}
    Since $\{x^\top\theta:x\in \cX,\theta\in \Theta\}\subset [S_2,S_1]$ by \cref{ass:model}, we have
    \begin{align*}
        \sup\{\scfunc(x^\top\theta):x\in \cX, \theta\in \Theta\}\le K.
    \end{align*}
    By \cref{lem:alpha_dotmu_relate}, we have that for all $x\in \cX$, 
    \begin{align*}
        \alpha(x^\top\theta_1,x^\top\theta_2) \ge
        (1+K|x^\top(\theta_1-\theta_2)|)^{-1}\dot\mu(x^\top\theta_1)\ge (1+2K\cdot (S_1-S_2))^{-1}\dot\mu(x^\top\theta_1).
    \end{align*}
    Then the following holds
    \begin{align*}
        \sum_{i=1}^t \alpha(X_i^\top\theta_1,X_i^\top\theta_2)X_iX_i^\top &\succeq (1+2K\cdot (S_1-S_2))^{-1}\sum_{i=1}^t \dot\mu(x^\top\theta_1)X_iX_i^\top\\
         G_t(\theta_1, \theta_2)&\succeq (1+2K\cdot (S_1-S_2))H_t(\theta_1),
    \end{align*}
    where the last inequality follows by $(1+2K\cdot (S_1-S_2))^{-1}\le 1$. The proof of  \cref{eq:GtHt_relate_2} follows by substituting $\theta_1$ with $\theta_2$.
\end{proof}
\ConfSetDiamH*

\begin{proof}
We first prove the statement for $\|\theta_1-\theta_2\|_{H_t(\theta_1)}$.
    By \cref{lem:GtHt_relate}, we have that
    \begin{align*}
        \|\theta_1-\theta_2\|_{H_t(\theta_1)} &\le \sqrt{(1+2K\cdot (S_1-S_2))}\|\theta_1-\theta_2\|_{G_t(\theta_1, \theta_2)}\\
        &=\sqrt{(1+2K\cdot (S_1-S_2))}\|g_t(\theta_1)-g_t(\theta_2)\|_{G_t^{-1}(\theta_1, \theta_2)}\tag{\cref{lem:MVT_application}}\\
        &\le \sqrt{(1+2K\cdot (S_1-S_2))}\Big(\|g_t(\theta_1)-g_t(\hat\theta_t)\|_{G_t^{-1}(\theta_1, \theta_2)}+ \|g_t(\hat\theta_t) - g_t(\theta_2)\|_{G_t^{-1}(\theta_1, \theta_2)}\Big)
    \end{align*}
    Note that $\theta_1, \theta_2 \in \cC_t^\delta(\hat\theta_t)$ by hypothesis, then \cref{lem:GtHt_relate} and the definition of $\cC_t^\delta(\hat\theta_t)$ gives that
    \begin{align*}
        \|g_t(\theta_1)-g_t(\hat\theta_t)\|_{G_t^{-1}(\theta_1, \theta_2)} & \le \sqrt{(1+2K\cdot (S_1-S_2))}\|g_t(\theta_1)-g_t(\hat\theta_t)\|_{H_t^{-1}(\theta_1)}\\
        &\le \sqrt{(1+2K\cdot (S_1-S_2))}\gamma_t(\delta)\\
        \|g_t(\hat\theta_t) - g_t(\theta_2)\|_{G_t^{-1}(\theta_1, \theta_2)} &\le \sqrt{(1+2K\cdot (S_1-S_2))}\|g_t(\hat\theta_t)-g_t(\theta_2)\|_{H_t^{-1}(\theta_2)}\\
        &\le \sqrt{(1+2K\cdot (S_1-S_2))}\gamma_t(\delta).
    \end{align*}
    Chaining all the inequalities together finishes the proof. The proof for the statement for $\|\theta_1-\theta_2\|_{H_t(\theta_2)}$ follows similarly by substituting $\theta_1$ with $\theta_2$. 
\end{proof}

\subsection{Proof of self-bounding property of self-concordance functions: \cref{lem:dot-mu-sum}}
As mentioned beforehand, the following lemma is abstracted out from Claim 14 of \citet{janz2023exploration}:
\label{sec:dot-mu-sum}
\DotMuSum*
\begin{proof}
We have
    \begin{align*}
        \sum_{t=1}^n \dot f(a_t) 
        &= \sum_{t=1}^n \dot f(b) + \sum_{t=1}^n (a_t-b)\int_{0}^1 \ddot f\left(b + v(a_t-b)\right)dv  \\
        &\leq n\dot f(b) + \sum_{t=1}^n \left|  (a_t-b)\int_{0}^1 \ddot f\left(b + v(a_t-b)\right)dv\right| \\
        &\leq  n\dot f(b) + \sum_{t=1}^n (b-a_t)\int_{0}^1 \left| \ddot f\left(b + v(a_t-b)\right) \right| dv 
        \tag{$a_t \leq b $ and triangle inequality} \\
        &\leq  n\dot f(b) + \sum_{t=1}^n(b-a_t)  \int_{0}^1 \scfunc(b + v(a_t-b))\dot f\left(b + v(a_t-b)\right)dv 
        \tag{\cref{lem:gdot_mu_rel}} \\
        &=  n\dot f(b) + \sum_{t=1}^n  (b-a_t)  \int_{0}^1 A\dot f\left(b + v(a_t-b)\right)dv \\
        &\le n\dot f(b) + K\sum_{t=1}^n f(b)-f(a_t)\,,\tag{fundamental theorem of calculus}
    \end{align*}
    finishing the proof.
\end{proof}

%% file: appendix/aux_lemmas.tex
\subsection{Auxiliary Lemma}\label{section:aux_lem}
\begin{lemma}[Elliptical potential lemma]\label{lem:epl}
    Fix $\lambda, A > 0$. Let $\{a_t\}_{t=1}^\infty$ be a sequence in $AB^d_2$ and let $V_0 = \lambda I$. Define $V_{t+1} = V_t + a_{t+1} a_{t+1}^\top$ for each $t \in \NN$. Then, for all $n \in \NN^+$,
    \begin{equation*}
        \sum_{t=1}^n \norm{a_t}_{V_{t-1}^{-1}}^2 \leq 2 d\max\left\{1,\frac{A^2}{\lambda}\right\} \log \left( 1 + \frac{n A^2}{d\lambda}\right) \,. 
    \end{equation*}
\end{lemma}
\begin{proof}
See, e.g.,  Lemma 19.4 of \citet{Lattimore_Szepesvári_2020}.
\end{proof}